\documentclass[a4paper,fleqn]{cas-dc}

\usepackage[numbers]{natbib}

\usepackage{algorithm}
\usepackage{algorithmic}
\usepackage{graphicx}
\usepackage{amsmath}
\usepackage{amsthm}
\usepackage{booktabs}
\usepackage{mathrsfs}
\usepackage{multirow}
\usepackage{colortbl}
\usepackage{booktabs}

\usepackage{bm}
\usepackage{color}
\usepackage{subfigure}
\usepackage{mathrsfs}
\usepackage{multirow}
\usepackage{threeparttable}
\newtheorem{theorem}{Theorem}

\newcommand{\ba}        {{\bf a}}

\newcommand{\bb}        {{\bf b}}
\newcommand{\be}        {{\bf e}}

\newcommand{\bE}        {{\bf E}}
\newcommand{\bX}        {{\bf X}}
\newcommand{\bx}        {{\bf x}}

\newcommand{\bh}        {{\bf h}}

\newcommand{\bW}        {{\bf W}}
\newcommand{\bw}        {{\bf w}}

\newcommand{\bmm}        {{\bf m}}
\newcommand{\bk}        {{\bf k}}

\newcommand{\bI}        {{\bf I}}
\newcommand{\bq}        {{\bf q}}

\def\tsc#1{\csdef{#1}{\textsc{\lowercase{#1}}\xspace}}
\tsc{WGM}
\tsc{QE}

\begin{document}
\let\WriteBookmarks\relax
\def\floatpagepagefraction{1}
\def\textpagefraction{.001}

\shorttitle{An Interpretable Graph Attentive Recurrent Neural Network for Predicting Blood Glucose Levels}    

\shortauthors{Chengzhe Piao {\em et al. }}  

\title [mode = title]{GARNN: An Interpretable Graph Attentive Recurrent Neural Network for Predicting Blood Glucose Levels via Multivariate Time Series}  



%

\author[1]{Chengzhe Piao}[orcid=0000-0003-0494-5098]
\ead{chengzhe.piao.21@ucl.ac.uk}
\credit{Conceptualization, Methodology, Software, Investigation, Writing – original draft}
\author[2]{Taiyu Zhu}
\ead{taiyu.zhu@psych.ox.ac.uk}
\credit{Data curation,  Validation, Formal analysis, Writing – review \& editing}
\author[3,4]{Stephanie E Baldeweg}
\ead{s.baldeweg@ucl.ac.uk}
\credit{Conceptualization, Writing – review \& editing}
\author[1]{Paul Taylor}
\ead{p.taylor@ucl.ac.uk}
\credit{Resources, Writing – review \& editing}
\author[5]{Pantelis Georgiou}
\ead{pantelis@imperial.ac.uk}
\credit{Data curation, Resources, Writing – review \& editing}
\author[6]{Jiahao Sun}
\ead{sun@flock.io}
\credit{Conceptualization, Writing – review \& editing}
\author[7]{Jun Wang}
\ead{junwang@cs.ucl.ac.uk}
\credit{Conceptualization, Writing – review \& editing, Supervision}
\author[1]{Kezhi Li}
\cormark[1]
\ead{ken.li@ucl.ac.uk}
\credit{Conceptualization, Writing – review \& editing, Project administration, Supervision, Resources, Funding acquisition}   
         
\affiliation[1]{organization={Institute of Health Informatics, University College London},
                city={London},
                postcode={NW1 2DA}, 
                country={UK}}
\affiliation[2]{organization={Department of Psychiatry, University of Oxford},
                city={Oxford},
                postcode={OX3 7JX}, 
                country={UK}}
\affiliation[3]{organization={Department of Diabetes \& Endocrinology, University College London Hospitals},
            city={London},
            postcode={NW1 2PG}, 
            country={UK}}
\affiliation[4]{organization={Centre for Obesity \& Metabolism, Dept of Experimental \& Translational Medicine, University College London},
            city={London},
            postcode={WC1E 6JF}, 
            country={UK}}
\affiliation[5]{organization={Centre for Bio-Inspired Technology, Department of Electrical and Electronic Engineering, Imperial College London},
            city={London},
            postcode={SW7 2AZ}, 
            country={UK}}

\affiliation[6]{organization={FLock.io},
            city={London},
            postcode={WC2H 9JQ}, 
            country={UK}}
\affiliation[7]{organization={Department of Computer Science, University College London},
            city={London},
            postcode={WC1E 6EA}, 
            country={UK}}












\cortext[1]{Corresponding author}



\begin{abstract}
Accurate prediction of future blood glucose (BG) levels can effectively improve BG management for people living with diabetes, thereby reducing complications and improving quality of life.
The state of the art of BG prediction has been achieved by leveraging advanced deep learning methods to model multi-modal data, i.e., sensor data and self-reported event data, organised as multi-variate time series (MTS). 
However, these methods are mostly regarded as ``black boxes'' and not entirely trusted by clinicians and patients.
In this paper, we propose interpretable graph attentive recurrent neural networks (GARNNs) to model MTS, explaining variable contributions via summarizing variable importance and generating feature maps by graph attention mechanisms instead of post-hoc analysis. 
We evaluate GARNNs on four dataset, representing diverse clinical scenarios.
Upon comparison with twelve well-established baseline methods, GARNNs not only achieve the best prediction accuracy {\color{black} but also provide high-quality temporal interpretability, in particular for postprandial glucose levels as a result of corresponding meal intake and insulin injection}. These findings underline the potential of GARNN as a robust tool for improving diabetes care, bridging the gap between deep learning technology and real-world healthcare solutions.
\end{abstract}


\begin{highlights}
\item We propose GARNNs to forecast blood glucose levels, providing prediction explanations by learning temporal variable importance.
\item GARNNs are based on graph attention to assess and prioritize temporal variables for the forecasting.
\item GARNNs offer insightful feature maps for the prediction, supported by theoretical proofs.
\item GARNNs are rigorously tested against twelve baselines, proving superior accuracy and interpretability for the prediction.
\item GARNNs' effectiveness across four datasets is demonstrated by their robustness for sparse input signals, such as insulin and meal intake.
\end{highlights}


\maketitle









\section{Introduction}
Diabetes is directly responsible for over a million deaths worldwide every year \cite{diabetes}
due to complications arising from type 1 diabetes mellitus (T1DM) and type 2 diabetes mellitus (T2DM). The autoimmune reaction of people with T1DM destroys the cells in the pancreas which produce endogenous insulin, while people with T2DM predominantly have insulin resistance, which inhibits their ability to utilize insulin effectively.
Difficulties to manage BG levels by endogenous insulin leads to hypoglycemia and hyperglycemia, causing serious health problems \cite{diacare, mora2023predicting}.
Hence, effective self-management for BG levels is the key to the treatment \cite{DBLP:journals/artmed/WoldaregayAWAMB19}, because increased ``Time in Range'' has been shown to reduce the likelihood of complications \cite{bezerra2023time}.

Continuous glucose monitoring (CGM) systems offer the ability to track BG levels every few minutes, generating continuous BG trajectories. 
BG level prediction (BGLP) based on CGM data \cite{DBLP:conf/aaai/PlisBMSS14,DBLP:journals/cbm/CichoszKJH21} allows people with diabetes to avoid hypo- and hyperglycemia by taking precautions in real-time. 
Recent work \cite{DBLP:journals/cbm/NematKEB23, 9681840,DBLP:journals/tbe/ZhuLHG23, DBLP:journals/npjdm/ZhuULHOG22,  DBLP:journals/cbm/KarimVK20} leverages 
multi-modal data by organizing it as MTS in BGLP.
In this case, apart from CGM data, the MTS input also
includes sensor data, e.g., heart rate, and self-reported events, e.g., the amount of carbohydrate intake and bolus insulin injection.
While these methods have the potential to further improve BGLP by levaraging the rich hidden information of MTS, the lack of interpretability makes them less trustworthy.
It is vital to understand how each variable contributes towards prediction rather than solely improving prediction accuracy. 

However, 
%
post-hoc analysis methods, e.g., gradient-based attribution
methods \cite{ancona2018towards}, are computationally inefficient \cite{DBLP:conf/icml/0002LA19} and difficult to be used by the researchers and clinicians without machine learning knowledge. 
%
{\color{black} Shapley Additive exPlanations (SHAP) values, as discussed by \citeauthor{{DBLP:conf/nips/LundbergL17}} in  \cite{DBLP:conf/nips/LundbergL17}, are utilized alongside Long Short-Term Memory (LSTM, \cite{DBLP:journals/neco/HochreiterS97}) models in BGLP \cite{DBLP:conf/ecai/CapponMPPSFF20, prendin2023importance}. It should be noted that while they offer significant insights into the importance of a limited number of variables, their scope in temporal variable importance is somewhat restricted.}

\begin{figure*}[tb]
	\centering
	\includegraphics[width = 2.0\columnwidth]{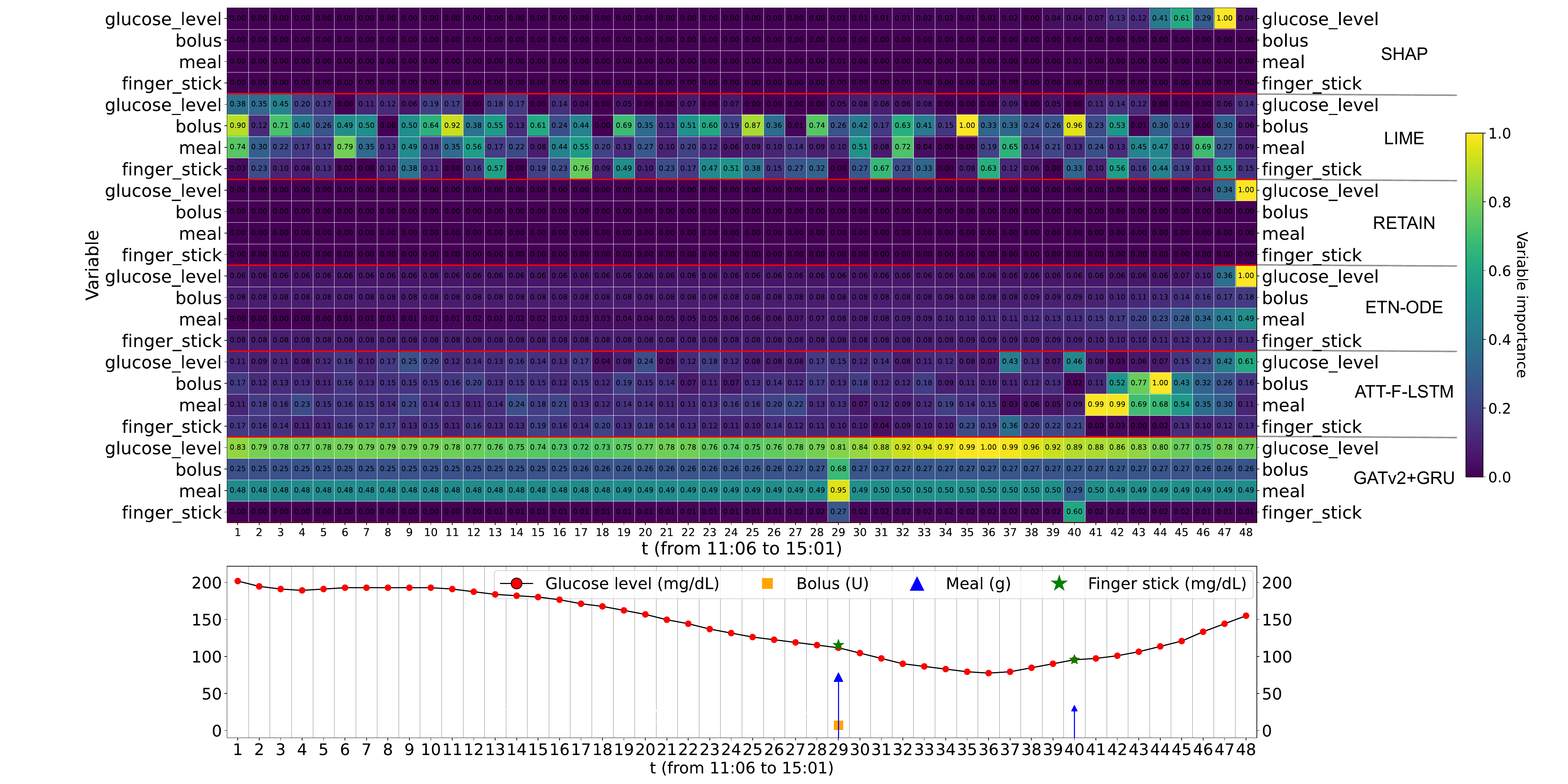}
	\caption{Some feature maps of participant 1005 in ArisesT1DM. The bottom sub-figure is the visualization of a historical multivariate time serires, only showing ``glucose\_level (continuous glucose monitoring)'', ``bolus'', ``meal'' and ``finger\_stick (capillary blood glucose test)''. The heat maps on the top are the feature maps of different methods. The x-axis/y-axis of heatmaps is the timestep/variable. The value in the cell is the variable importance, scaled to $[0,1]$. {\color{black} Abbr.: SHAP (SHapley Additive exPlanations, \cite{DBLP:conf/nips/LundbergL17}), LIME (Local Interpretable Model-agnostic Explanations, \cite{DBLP:conf/kdd/Ribeiro0G16}), RETAIN (REverse Time AttentIoN, \cite{DBLP:conf/nips/ChoiBSKSS16}) , ETN-ODE (Explainable Tensorized Neural Ordinary Differential Equations, \cite{9757812}), ATT-F-LSTM (ATTention of Feature before Long Short-Term Memory \cite{gandin2021interpretability}) and our proposed method ``GATv2+GRU'' (Graph Attention NeTworks version 2 \cite{DBLP:conf/iclr/Brody0Y22} and Gated Recurrent Unit \cite{DBLP:conf/ssst/ChoMBB14}).}}
	\label{fig:244}
\end{figure*}

Comparably, attention-based recurrent neural networks (RNNs), e.g., RETAIN \cite{DBLP:conf/nips/ChoiBSKSS16}, ETN-ODE \cite{9757812} and ATT-F-LSTM \cite{gandin2021interpretability}, can inherently learn variable importance by the attention mechanisms during training.
Nevertheless, the feature maps built on the variable importance are unhelpful to understand BGLP.
Hence, we aim to propose a novel attention-based interpretable model that can rank variables in accordance with domain knowledge and generate understandable feature maps. 
{\color{black}
For instance, figure \ref{fig:244} demonstrates the comparison of existing interpretable methods in interpreting the time-dependent importance of self-reported events for BG trajectories.}
%
%
{\color{black}
Other interpretable methods fail to present their correct focus on self-reported events, as shown in their feature maps.
On the contrary}, our proposed method ``GATv2+GRU'' provides insights into the importance of variables when valid observations are available for those variables. For instance, in our feature map, the importance of ``bolus'' suddenly increases {\color{black} from $0.27$} to $0.68$ at $t=29$, because this participant administrated bolus insulin at this timestep.
%
{\color{black}
When data is incomplete or invalid, like when missing points are filled with average values, the importance of variables tends to become stable at certain numbers. In such cases ($t \neq 29$), the importance of ``bolus'' typically remains near $0.26$, reflecting an average importance due to this data padding.}
Besides, it can precisely capture the local maxima and minima. When $t=36$, the ``glucose\_level'' achieves its lowest local minima, and its importance increases to $1.0$. 

{\color{black} When compared with both non-interpretable and interpretable approaches, our proposed methods outperform others in BGLP. Additionally, they offer effective explanations for MTS by inherently providing detailed insights. To summarize, the contributions of this paper are as follows:}

\begin{itemize}
    \item We propose Graph Attentive Recurrent Neural Networks, denoted as GARNNs, combining Graph Attention neTworks (GAT \cite{DBLP:conf/iclr/VelickovicCCRLB18} or GATv2 \cite{DBLP:conf/iclr/Brody0Y22}), with RNNs. {\color{black} We leverage GAT/GATv2 to explicitly model correlations among various variables, resulting in the inherent learning of temporal variable importance. Subsequently, RNNs are employed to aggregate temporal features for the prediction of future BG levels.}
    \item We propose a novel temporal variable importance for MTS based on graph attention mechanisms. It can respectively summarize and generate significant variable importance ranking and feature maps in BGLP. {\color{black} We employ three theorems and their corresponding proofs to demonstrate the creation, characteristics, and influence of temporal variable importance.}
    \item {\color{black} 
Our methods, along with twelve baseline approaches, are assessed using four datasets.
This evaluation focuses on the accuracy of predictions and the significance of variable ranking and feature maps. Not only do our methods exhibit outstanding performance, but they also attribute medically justifiable significance to crucial variables, including those with sparse signals like insulin injections and meal intake.}
\end{itemize}

\section{Related Work}
Interpretable methods can be generally divided into post-hoc analysis methods and attention-based methods.

If a model itself lacks interpretability, it is still possible to derive explanations via post-hoc analysis methods, e.g., gradient-based analyzing \cite{DBLP:conf/icml/ShrikumarGK17, DBLP:conf/aaai/BykovHNH22}, variable contribution to the temporal shift of model \cite{tonekaboni2020went}, Local Interpretable Model-agnostic Explanations (LIME, \cite{DBLP:conf/kdd/Ribeiro0G16}). SHapley Additive exPlanations (SHAP, \cite{DBLP:conf/nips/LundbergL17}), shapley based feature attributions \cite{kwon2022weightedshap} and local rule based explanation \cite{DBLP:conf/aaai/RajapakshaB22}.

However, considering the efficiency and the target to explain the modeling of each MTS, the above methods are not considered in BGLP.
Instead, we focus on attention-based methods, where variable importance based on attention is inherently learned during training.

Tensorized RNN methods \cite{DBLP:conf/icml/0002LA19,DBLP:conf/icdm/ChuWMJZY20, DBLP:journals/titb/ShamoutZSWC20,DBLP:conf/icml/AguiarSW022, 9757812} leveraged parallel RNNs to model MTS, where each time series was modeled by a certain RNN.
Then, they used attention to aggregate the outputs of parallel RNNs.
In this regard, IMV-TENSOR \cite{DBLP:conf/icml/0002LA19} leveraged variable-wised temporal attention and variable attention to summarize all the outputs of parallel LSTMs, where the attention was directly mapped from the outputs of LSTMs by fully-connected layers.
Similarly, instead of using parallel RNNs, \citeauthor{DBLP:conf/wsdm/HsiehWSH21} utilized parallel 1-dimensional convolution neural networks for modeling MTS aided by attention \cite{DBLP:conf/wsdm/HsiehWSH21}.

These parallel structures explicitly split the contribution of variables in the prediction, but they ignored the interaction among the variables.
Besides, extracting variable importance from the temporal attention which worked on the outputs of RNNs could cause temporal biases.
For example, as shown in Figure \ref{fig:244}, we cannot infer the variable importance of ``meal'' from the feature map of ETN-ODE, when $t=29$.
This is because some important content at $t=29$ is passed to the last several timesteps by RNNs, and feature maps are built on the outputs of RNNs.
{\color{black}
While RETAIN \cite{DBLP:conf/nips/ChoiBSKSS16} and RAIM \cite{DBLP:conf/kdd/XuBDMS18} did not incorporate parallel structures and utilized the attention mechanism to combine embeddings of MTS instead of RNN outputs, their attention process was still influenced by the outputs from RNNs.
For instance, RETAIN employed two separate RNNs to create attention at both the visit and variable levels. This RNN-derived attention was then applied to aggregate the MTS embeddings. Consequently, in RETAIN, the significance assigned to variables was somewhat biased by the RNNs, as illustrated in Figure \ref{fig:244}.}

{\color{black} Even though certain researchers had attempted to minimize the influence of RNNs on the importance of temporal variables, the effectiveness of this importance remained uncertain. \citeauthor{kaji2019attention} and \citeauthor{gandin2021interpretability} had implemented attention mechanisms prior to the RNN layers, ensuring that the variable importance was not skewed by the RNNs. However, this approach, which was directed either by variable-wise temporal attention \cite{kaji2019attention} or time-wise variable attention \cite{gandin2021interpretability}, still led their models to erroneously focus on some irrelevant data points, particularly in the context of sparse signals. This issue is evident in models like ATT-F-LSTM \cite{gandin2021interpretability}, as showcased in Figure \ref{fig:244}. }

\section{Preliminaries}

Given a directed graph $\mathcal{G} \triangleq ({\mathcal{N},\mathcal{E}}, \bE)$, we assume that the graph $\mathcal{G}$ has $N$ nodes, denoted as $\mathcal{N}\triangleq\{1, ..., n, ..., N\}$.
Each node $n$ is connected with its neighborhood $\mathcal{N}^n\subseteq \mathcal{N}$, through the edges of $\mathcal{E}$.
Besides, the presentations of all the nodes are $\bE=[\be^1\ ...\ \be^n\ ... \ \be^N]\in\mathbb{R}^{E\times N}$, where the presentation of the node $n$ is $\be^n$.
Assuming that there are multiple layers in  Graph Attention neTworks (GATs), defined as $\mathcal{L}\triangleq\{1, ..., l, ..., L\}$, each node $n$ in layer $l$ can receive and aggregate neural messages from $\mathcal{N}^n$.
Then, one of the aggregation approaches is:
\begin{align}
    \label{Eqn:aggre}
    &\be^{n,l+1} = \sum_{j\in\mathcal{N}^n}\Tilde{s}^{n,j, l}(\bW^l\be^{j,l}+ \bb^l),\\
    \label{Eqn:att}
    &\Tilde{s}^{n,j,l} = {\rm Softmax} (s^{n,j,l}),
\end{align}
where the neural message from node $j\in\mathcal{N}^n$ is $\bW^l\be^{j,l}+ \bb^l$; learnable parameters are $\bW^l$ and $\bb^l$; the attention weight from $j$ to $n$ is $\Tilde{s}^{n,j, l}\in\mathbb{R}^1$.

We temporarily omit the superscript $l$ for simplicity. 
The attention weight $\Tilde{s}^{n,j}$ is gotten by normalizing the score $s^{n,j}$.
The score from node $j$ to node $n$ is $s^{n,j}$, as:
\begin{align}
    \label{Eqn:gat}
    &{\rm GAT}:\ s^{n,j} = {\rm LeakyReLU}
    (\ba^{\prime \top}[\bW\be^n+\bb;\bW\be^j+\bb]),\\
    \label{Eqn:gatv2}
    &{\rm GATv2}: \ s^{n,j} = \ba^{ \top} {\rm LeakyReLU}(\bW^\prime[\be^n;\be^j] + \bb),
\end{align}
where $\ba\in\mathbb{R}^{A}$ and $\ba^\prime \in\mathbb{R}^{2A}$ are learnable parameters; {\color{black} constant $A$ is a hyperparameter}; the concatenation of vectors is $[;]$.

The above score function can be rewritten as:
\begin{align}
    \label{Eqn:gat_ext}
    &{\rm GAT}:\ s^{n,j} = {\rm f} (\bq^n, \bk^j)= {\rm LeakyReLU}\left(\ba_1^\top\bq^n+\ba_2^\top\bk^j\right),\\
    \label{Eqn:gatv2_ext}
    &{\rm GATv2}:\ s^{n,j} = {\rm f} (\bq^n, \bk^j) = \ba^\top {\rm LeakyReLU}(\bq^n+ \bk^j),
\end{align}
where $\ba_1,\ba_2\in\mathbb{R}^{A}$. The query $\bq^n$ is $\bW_1\be^n+\bb_1$, and the key $\bk^j$ is $\bW_2\be^j + \bb_2$.
As for Equation (\ref{Eqn:gat_ext}), the learnable parameters $\{\bW_1, \bb_1\}$ should be the same as $\{\bW_2,\bb_2\}$.

{\color{black} 
Based on the findings presented in \citep{DBLP:conf/iclr/Brody0Y22}, the scoring mechanism in GAT and GATv2 differs, with GAT having a static approach and GATv2 featuring a dynamic one.
As a result, GATv2 demonstrates greater expressiveness in modeling intricate features of graphs, offering enhanced capabilities compared to its predecessor.}

\textbf{Static scoring}: 
given queries $\{\bq^n| n\in\mathcal{N}\}$ and keys $\{\bk^m|m\in\mathcal{M}\}$, it holds ${\rm f}(\bq^n, \bk^m) \leq {\rm f}(\bq^n, \bk^{m^\prime=m^{max}})$, where $\forall n \in \mathcal{N}$,   $\forall m \in \mathcal{M}$, $\exists m^\prime \in \mathcal{M}$, ${\rm f}: \mathbb{R}^A\times\mathbb{R}^A\rightarrow\mathbb{R}^1$.

\textbf{Dynamic scoring}: given queries $\{\bq^n| n\in\mathcal{N}\}$ and keys $\{\bk^m|m\in\mathcal{M}\}$, it holds ${\rm f}(\bq^n, \bk^m) < {\rm f}(\bq^n, \bk^{m^\prime=\phi(n)})$, where $\forall n \in \mathcal{N}$, $\exists m^\prime \in \mathcal{M}$, $\forall m \in \mathcal{M}$ and $m\neq\phi(n)$. Meanwhile, function $\phi: \mathcal{N}\rightarrow\mathcal{M}$ and ${\rm f}: \mathbb{R}^A\times\mathbb{R}^A\rightarrow\mathbb{R}^1$.

\section{Proposed model}
\subsection{Problem definition}
\textbf{Blood glucose level prediction based on multivariate time series (BGLP-MTS)}: given the values of variables $\mathcal{N}$ from historical timesteps $\mathcal{T}\triangleq\{1, ..., t, ..., T\}$, i.e., $\bX=[\bx_1 ...\ \bx_t ...\ \bx_T ]\in \mathbb{R}^{N\times T}$, predict the BG level $y_{T+H}$. The vector $\bx_t = [x_t^1\ ...\ x_t^n\ ...\ x_t^N]^\top$, and $H$ is a prediction horizon. 
We let $n=1$ be the target variable, i.e., $x^1_t = y_t$.
The rest variables ($n>1$) are exogenous variables when $N>1$.

\subsection{Overview}

\begin{figure}[tb]
	\centering
	\includegraphics[width = 1.0\columnwidth]{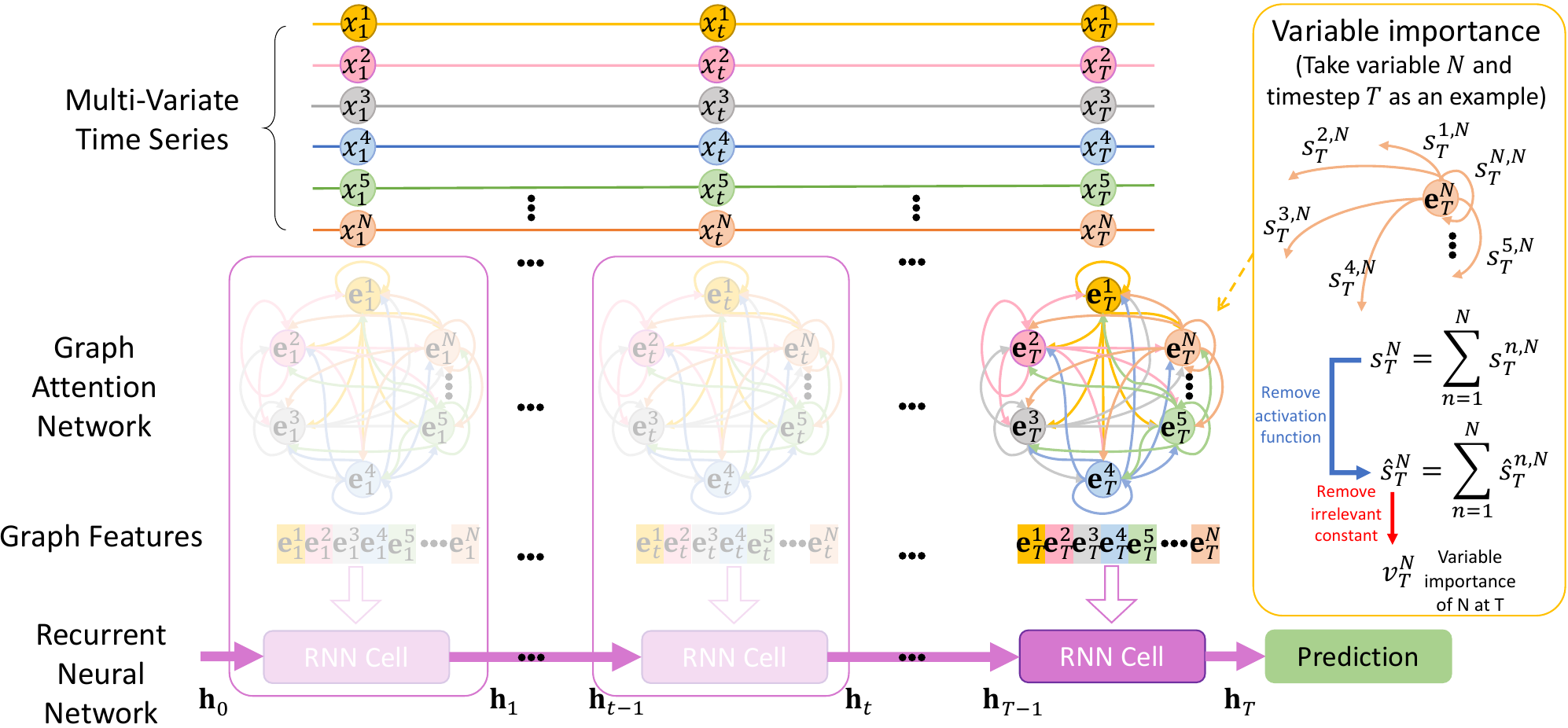}   
    \caption{Graph attentive recurrent neural networks (GARNNs). The observation of the variable $n\in\mathcal{N}\triangleq\{1,...,n,...,N\}$ at timestep $t\in\{1,...,t,...,T\}$ is $x_t^n$. The total length of the historical multivariate time series is $T$. The attention score from $j$ to $n$ is $s_t^{n,j}$ ($j\in \mathcal{N}$) for aggregating neural messages at node $n$.
    The variable importance of $j$ is $v^j_t$ which is gotten from $s^j_t$. 
    The hidden state is $\bh_t$. }
	\label{fig:model}
\end{figure}

Our proposed models, GARNNs, build a graph at each timestep and use each node $n$ of the graph $\mathcal{G}$ to represent a variable $n$ of MTS, assuming each graph is initially a complete graph (see Figure \ref{fig:model}).
Then, the input $\be_t^n$ of Equation (\ref{Eqn:aggre}) is 

\begin{equation}
\label{Eqn:emb}
    \be_t^n = {\rm ReLU}(\bw^n x_t^n + \bb^n),
\end{equation}
where learnable parameters $\bw^n \in \mathbb{R}^{E}$.
Then, we can use Equation (\ref{Eqn:aggre}-\ref{Eqn:att}) aided by Equation (\ref{Eqn:gat_ext}) or Equation (\ref{Eqn:gatv2_ext}) to update representations of $n$.
Next, we collect the latest representations $\be_t^{1:N}$ and concatenate them as $\be_t = [\be_t^1; ...;\be_t^n;...; \be_t^N]$.

After explicitly modeling correlations of these variables, we collect $\be_t$ of all timesteps, denoted as $\be_{1:T}$.
Then, we leverage RNN to aggregate them, as:
\begin{equation}
\label{eqn:gru}
    \bh_{1:T} = {\rm RNN}(\be_{1:T}, \bh_{0:T-1}),
\end{equation}
where we utilize gated recurrent unit (GRU, \cite{DBLP:conf/ssst/ChoMBB14}) as ${\rm RNN}(\cdot)$ to aggregate temporal features in this paper.
Finally, the prediction is $\hat{y}_{T+H}={\rm MLP}(\bh_T)$,
where ${\rm MLP}(\cdot)$ consists of fully connected neural networks.

Given training examples $\mathcal{I}\triangleq\{1,...,i,...,I\}$, the objective function is:
\begin{equation}
\label{eqn:obj_fun}
    J(\theta) = \frac{1}{I}\sum_i\left(\hat{y}_{T+H}(i) - y_{T+H}(i)\right)^2 + \frac{\lambda}{2}\Vert{\theta}\Vert^2_2,
\end{equation}
where $\theta$ are all the learnable parameters of our proposed model; $\lambda$ is a hyperparameter.

\subsection{Interpretability}

\begin{theorem}
\textit{There exists a universal interpretable variable importance of GAT and GATv2}
.
\end{theorem}

\begin{proof}
{\color{black}
In BGLP scenario, we assume each node $n$ connects with all of the nodes of $\mathcal{N}$. Then, as for the neighbour nodes $\mathcal{N}^n$ of the node $n$, we have $\mathcal{N}^n = \mathcal{N}$.}
Considering Equation (\ref{Eqn:aggre}-\ref{Eqn:att}) and Equation (\ref{Eqn:gat_ext}-\ref{Eqn:gatv2_ext}), for the variable $n$, the higher the score $s_t^{n,j}$, the more important $j$ to $n$.
The mean of the scores from $j$ to all variables of $\mathcal{N}$ is denoted as:
\begin{equation}
    \label{Eqn:mean_score}
    s_t^j = \frac{1}{N}\sum_n s^{n,j}_t.
\end{equation}
It is the impact of the variable $j$ on all the variables.
{\color{black}
Theoretically, a higher level of importance assigned to variable $j$ corresponds to an increased mean score $s_t^j$. 
As a result, an enhanced amount of information from $j$ is conveyed to and preserved within the embedding $\be_t$ collected from GAT/GATv2, thereby amplifying the contribution of $j$ in the prediction process.
}

However, we do not regard $s_t^j$ as the variable importance of $j$ at timestep $t$.
Instead, we extract its variable importance from $s_t^j$ by removing irrelevant information.
Specifically, considering that ${\rm LeakyReLU}(x)$ is monotonic with respect to $x$, we remove ${\rm LeakyReLU}(\cdot)$ from Equation (\ref{Eqn:mean_score}).
GAT and GATv2 can be organised in the same format, as:
\begin{align}
    \label{eqn:vari_imp_1}
    \hat{s}_t^j &= \frac{1}{N}\sum_n \hat{s}^{n,j}_t = \frac{1}{N}\sum_n (\ba_1^\top\bq^n_t + \ba_2^\top\bk^j_t), \\
    \label{eqn:vari_imp_2}
    &=\underbrace{\ba_2^\top \bk^j_t}_\text{variable importance} + \frac{1}{N}\sum_n \ba_1^\top\bq^n_t
\end{align}
where $\ba_1$ is unequal and equal to $\ba_2$ in GAT and GATv2 repectively.
Then, we define the variable importance $v^j_t$ by removing the irrelevant item $\frac{1}{N}\sum_n \ba_1^\top\bq^n_t$ from $\hat{s}_t^j$:
\begin{equation}
\label{Eqn:vari_imp}
    v^j_t \triangleq  \ba_2^\top \bk^j_t = \underbrace{\ba_2^\top  \bW_2}_\text{variable contribution} \underbrace{\be^j_t}_\text{variable embedding} + \underbrace{\ba_2^\top\bb_2}_\text{constant bias}.
\end{equation}

On the other hand, we can consider multiple layers of GAT or GATv2, the variable importance $v^j_t$ is defined as:
\begin{align}
    v^j_t = \frac{1}{L}\sum_l v_t^{j,l}, \ \ v^{j,l}_t \triangleq {\ba_2^{l}}^{\top} \bk^{j,l}_t.
\end{align}
Given $\mathcal{I}$, the variable importance of $v_j$ over $\mathcal{I}$ is:
\begin{equation}
\label{Eqn:v_imp_all}
    v^j(\mathcal{I}) = \frac{1}{IT}\sum_{i,t} v_t^j(i),
\end{equation}
where $v_t^j(i)$ is the variable importance of $j$ at timestep $t$ of the $i$-th sample;
{\color{black}
the total number of training examples is $I$;
the length of the MTS is $T$.
}
\end{proof}

As in Equation (\ref{Eqn:vari_imp}), our proposed variable importance $v_t^j$ is fully understandable.
It consists of a variable contribution $\ba_2^\top  \bW_2$, a variable embedding $\be_t^j$ and a constant bias $\ba_2^\top \bb_2$.
The variable contribution and constant bias are learnable, so $v_t^j$ is directly guided by the variable embedding $\be^j_t$.
When $L=1$, the variable importance $v^j(\mathcal{I})$ is only affected by $j$.
When $L>1$, the variable importance $v^j(\mathcal{I})$ considers correlations of $j$ and other variables.

\begin{theorem}
The variable importance $v^j_t$ is from static scoring.
\end{theorem}

\begin{proof}
    Considering that $ \hat{s}^{n,j}_t = \ba_1^\top\bq^n_t + \ba_2^\top\bk^j_t$, when $n$ and $t$ is fixed, $\ba_1^\top\bq^n_t$ can be seen as a constant.
    $\hat{s}^{n,j}_t$ is largely affected by $\bk_t^j$.
    There can exist a $j^\prime \in \mathcal{N}$, making $\hat{s}^{n,j}_t\leq\hat{s}^{n,j^\prime=j^{max}}_t$.
    This observation holds for any $n$ if $t$ is fixed.
    The variable importance $v^j_t$ is extracted from $\hat{s}^{j}_t$ which is the mean of $\hat{s}^{n,j}_t$ over $n$, so $v^j_t$ is from static scoring.
    \end{proof}

    
{\color{black}
The use of static scoring for determining variable importance aligns with our expectations, primarily because dynamic scoring does not guarantee consistent significance of variables. 
For instance, in the case of GATv2, if we arrange variable $j$ based on the value of $s_t^{n,j}$ among $\{s_t^{n,j}| j\in\mathcal{N}\}$, the dynamic scoring leads to considerable fluctuations in the ranking of $j$ within $\mathcal{N}$ as $n$ varies. Consequently, using $s_t^j$ from GATv2 as a measure of variable importance becomes unreliable.

In contrast, the ranking of $j$ in $ \mathcal{N}$ based on the value of  $\hat{s}_t^{n,j}$ in $\{\hat{s}_t^{n,j}| j\in\mathcal{N}\}$ remains constant despite changes in $n$. Averaging $\hat{s}_t^{n,j}$ over $n$ to reevaluate the ranking of $j$ does not alter its position, emphasizing the necessity for static scoring in assessing variable importance.
Furthermore, eliminating $\frac{1}{N}\sum_n \mathbf{a}_1^\top\mathbf{q}^n_t$ from $ \hat{s}_t^{j}$ also maintains the ranking of $j$, supporting this approach.

Additionally, the implementation of variable importance based on static scoring does not interfere with the dynamic scoring function of GATv2 in mapping out variable correlations. This ensures that GATv2 remains both effective in modeling and consistent in calculating variable importance.
}
    
\begin{theorem}
\label{the:bounds}
The difference between $\hat{s}_t^j$ and $s_t^j$ is bounded by small values depending on the slope $\alpha\in[0,1]$ of ${\rm LeakyReLU}(\cdot)$.
\end{theorem}

\begin{proof}
In terms of GATv2 (see Equation (\ref{Eqn:gatv2_ext})), we have:
\begin{equation}
    s_t^j = \frac{1}{N}\sum_n \ba^\top \Tilde{\bI}^{n,j}_t (\bq^n_t+\bk^j_t)=  \frac{1}{N}\sum_n \ba^\top \Tilde{\bI}^{n,j}_t \bmm_t^{n,j},
\end{equation}
where $\Tilde{\bI}^{n,j}_t$ is an indicate diagonal matrix, as:
\begin{equation*}
\begin{split}
    &\Tilde{\bI}^{n,j}_t = {\rm Diag}(i^{n,j}_{t,1}, \cdots, i^{n,j}_{t,a}, \cdots, i^{n,j}_{t,A} ),\\
 &i^{n,j}_{t,a}= \begin{cases}
1,\quad &m_{t,a}^{n,j}\geq 0 \\
\alpha,\quad &m_{t,a}^{n,j}<0, \alpha \in [0, 1],
\end{cases} 
\end{split}
\end{equation*}
and $m_{t,a}^{n,j}$ is the $a$-th value of the vector $\bmm_{t}^{n,j}$.
\begin{align}
    s_t^j - \hat{s}_t^j &= \frac{1}{N}\sum_n \ba^\top \Tilde{\bI}^{nj}_t \bmm_t^{n,j} - \frac{1}{N}\sum_n \ba^\top \bmm_t^{n,j},\\
   &= \Vert \ba \Vert_2 \left \Vert \frac{1}{N}\sum_n ( \Tilde{\bI}^{n,j}_t - \bI  )\bmm_t^{n,j} \right \Vert_2 \cos{\beta},
\end{align}
where vetorial angle is $\beta$.
The boundary of the difference between $\hat{s}_t^j$ and $s_t^j$ is:
\begin{equation}
    0\leq \lvert s_t^j - \hat{s}_t^j \lvert \leq \Vert \ba \Vert_2 \left \Vert \frac{1}{N}\sum_n ( \Tilde{\bI}^{n,j}_t - \bI)\bmm_t^{n,j} \right \Vert_2.
\end{equation}

Equation (\ref{eqn:obj_fun}) has a soft constraint for $\Vert\theta\Vert_2 < \epsilon$. Based on the theory of Lagrange multipliers, there exists a $\lambda$ value that is equivalent to the hard constraint for $\Vert\theta\Vert_2 < \epsilon $, where $\epsilon$ is a small positive value. Then, for a learnable parameter $w$, we can have $\lvert w\lvert<\epsilon, \forall w\in\theta$. 
%
%
We assume each value of the input vector $\be^{j,l}_t$ of each GATv2 layer belongs to $[-c, c]$.
Given that $\Vert \ba \Vert_2 < \sqrt{A}\epsilon$ and $\left\Vert \frac{1}{N}\sum_n ( \Tilde{\bI}^{n,j}_t - \bI  )\bmm_t^{n,j} \right \Vert_2 \leq (1-\alpha)2\sqrt{A}(Ec + 1)\epsilon$, we have: 
\begin{equation}
     0 \leq \vert s_t^{j} - \hat{s}_t^{j} \vert  \leq (1-\alpha)2A(Ec + 1)\epsilon^2.
\end{equation}

Similarly, in terms of GAT, we have:
\begin{equation}
    0 \leq s_t^{j} - \hat{s}_t^{j}\leq (1-\alpha)2A(Ec + 1)\epsilon^2.
\end{equation}


\end{proof}

The bounds can be treated as the gaps between the modeling for prediction and the calculation of variable importance.
The gaps are affected by $\alpha$. Nevertheless, changing $\alpha$ has a slight impact on variable importance (see Figure \ref{fig:variable_importance_alpha}).
Hence, we hold the view that the small gaps between the prediction and the variable importance can be ignored.

\section{Experiments}

\subsection{Datasets}
We select four datasets containing multi-modal data with high quality, where T1DM and T2DM groups in Shanghai dataset \citep{zhao2023chinese} are regarded as two datasets.
{\color{black}
All datasets used in the study, except for ArisesT1DM (NCT ID: NCT03643692), are publicly accessible. The ArisesT1DM dataset was gathered in full compliance with relevant legal regulations. Consequently, ethical approval is not further required for this research.
}

\textbf{OhioT1DM \citep{DBLP:conf/ecai/MarlingB20}}: it is a well-validated dataset and has been used as a benchmark dataset in BGLP Challenges.
It has 12 participants (age: 20-60 years; male/female: 7/5), containing 8 weeks' CGM (Medtronic Enlite), sensor band (Empatica or Basis Peak) and self-reported data (meal, bolus, sleep, exercise, etc.).
Each participant has more than 10,000 CGM measurements, and BG readings are sampled every 5 minutes ($\delta t = 5\ {\rm min}$).

\textbf{ShanghaiT1DM and ShanghaiT2DM \citep{zhao2023chinese}}: it is a public dataset, consisting of 12 participants (age: 37-73 years; male/female: 7/5) with T1DM and 100 participants (age: 22-97 years, male/female: 44/56) with T2DM. It has 3-14 days' CGM (FreeStyle Libre H) and self-reported data.
There are 15,695 and 112,475 CGM measurements of people with T1DM and T2DM respectively, and BG readings are sampled per 15 minutes ($\delta t = 15\ {\rm min}$).

\textbf{ArisesT1DM (NCT ID: NCT03643692)}: it has 12 participants (age: 30–49 years; male/female: 6/6) and 6 weeks' CGM (Dexcom G6), sensor band (Empatica E4) and self-reported data.
Each participant has more than 10,000 CGM readings, and BG readings are sampled every 5 minutes ($\delta t = 5\ {\rm min}$).
In order to simplify the exhibition of experiment results, we leverage abbreviations of variable names \cite{DBLP:journals/npjdm/ZhuULHOG22}, e.g., electrodermal activity (EDA).   

OhioT1DM has been originally divided into training data and testing data by time for each participant.
We further split the training data into the training part (80\%) and validation part (20\%).
In terms of the rest three datasets, we respectively split the data into training data (60\%), validation data (20\%) and testing data (20\%) by time for each participant.
We use sliding windows to generate examples $(\bX,y_{T+H})$, i.e., $T=48$ and $H=6$ for $\delta t = 5\ {\rm min}$, or $T=16$ and $H=2$ for $\delta t = 15\ {\rm min}$.
All the examples are normalized by standard normalization and padded by zeros.
We consider almost all the variables (see Figure \ref{fig:variable_importance}) in each dataset. However, certain variables with poor data quality, such as ``stressors" in the OhioT1DM dataset, were excluded from our analysis.

\subsection{Baselines}
\label{sec:baseline}
\textbf{SHapley Additive exPlanations (SHAP, \cite{DBLP:conf/nips/LundbergL17}) and Local Interpretable Model-agnostic Explanations (LIME, \cite{DBLP:conf/kdd/Ribeiro0G16})}: they are model-agnostic methods, providing variable importance for any methods. Given that our proposed methods got the best predicting performance compared with all of baseline methods, we leverage SHAP and LIME to explain our proposed methods for comparisons. We directly treat the absolute value of the  variable importance provided by SHAP/LIME as $v^j$, and $v^j(\mathcal{I})$ is gotten by Equation (\ref{Eqn:v_imp_all}).

\textbf{Linear Regression (LR)}: it is a linear method.
We flatten the input $\bX$ and use scikit-learn \citep{scikit-learn} to fit models.
We aggregate the coefficients of the model as $v^j(\mathcal{I})$.

\textbf{eXtreme Gradient Boosting (XGBoost, \cite{Chen:2016:XST:2939672.2939785})}: it is an optimized distributed gradient boosting approach. We flatten the input $\bX$ to fit models. 
We regard the average gain as $v^j(\mathcal{I})$, where the gain is collected across all splits when using variables.

\textbf{REverse Time AttentIoN (RETAIN, \citep{DBLP:conf/nips/ChoiBSKSS16})}: it is an interpretable RNN model. The variable importance is calculated by the outputs of two RNNs, some learnable parameters and the input value of a variable. We use the absolute value of the variable importance of RETAIN as $v^j_t$, and $v^j(\mathcal{I})$ is gotten by Equation (\ref{Eqn:v_imp_all}). 

\textbf{Interpretable
Multi-Variable Long Short-Term Memories (IMV-LSTMs, \citep{DBLP:conf/icml/0002LA19})}: both \textbf{IMV-TENSOR} and \textbf{IMV-FULL} are interpretable LSTMs by generating variable importance and variable-wise temporal importance.
We directly average the variable importance of IMV-LSTMs with $\mathcal{I}$ examples and regard the mean variable importance as $v^j(\mathcal{I})$.
Meanwhile, we treat the variable-wise temporal importance as $v^j_t$.
%

\textbf{Explainable Tensorized Neural Ordinary Differential Equations (ETN-ODE, \citep{9757812})}: it consists of: 1) Tensorized GRU; 2) tandem attention; 3) ordinary differential equation network. 
Part 1 and 2 are for the interpretation, which is similar as IMV-TENSOR, and part 3 is for the arbitrary-step prediction.
%

\textbf{ATTention of Time series before Long Short-Term Memory (ATT-T-LSTM, \citep{kaji2019attention})}: it separately adds temporal attention for each variable before passing through LSTM. We regard the temporal attention as $v_t^j$, while $v^j(\mathcal{I})$ cannot be calculated by this method. 

\textbf{ATTention of Features before Long Short-Term Memory (ATT-F-LSTM, \citep{gandin2021interpretability})}: it adds variable attention at each timestep before LSTM. Hence, the attention weight of a variable $j$ at timestep $t$ is $v_t^j$,  and $v^j(\mathcal{I})$ is calculated via Equation (\ref{Eqn:v_imp_all}).

\textbf{Neural Basis Expansion Analysis for
Interpretable Time Series forecasting (N-BEATS, \cite{DBLP:conf/iclr/OreshkinCCB20})}: it is a non-interpretable deep learning model for univariate time series modeling. It utilizes a stack of fully connected neural network layers to model time series data, offering remarkable accuracy and flexibility.

\textbf{Neural Hierarchical Interpolation for Time Series forecasting (NHiTS, \cite{DBLP:conf/aaai/ChalluOORCD23})}: it is a non-interpretable deep learning model for MTS modeling, extending the N-BEATS and performing better in long-horizon prediction.

{\color{black}
Our proposed method Graph Attentive Recurrent Neural Net-
works (\textbf{GARNNs}) are represented by Graph Attention neTworks (GAT \cite{DBLP:conf/iclr/VelickovicCCRLB18}) or GATv2 \cite{DBLP:conf/iclr/Brody0Y22}) and Gated Recurrent Unit (GRU, \cite{DBLP:conf/ssst/ChoMBB14}), i.e., \textbf{``GAT+GRU''} and \textbf{``GATv2+GRU''}.}

All the deep methods are implemented by PyTorch 1.11.0 following their original codes on github and run with NVIDIA RTX 3090 Ti. 
All the methods except LR are trained four times by changing the random seed.
In terms of XGBoost, we search for the learning rate in $\{0.01, 0.1, 1.0\}$, n\_estimators in $\{50,100,200\}$, max\_depth in $\{3, 4, 5, 6, 7\}$, gamma in $\{0.5, 1, 1.5, 2, 5\}$ and min\_child\_weight in $\{1, 5, 10\}$.
For all the deep methods, we search for the learning rate in $\{10^{-3}, 10^{-4}, 10^{-5}\}$.
For IMV-LSTM and ETN-ODE, we find the variable-wise hidden state size in $\{8, 16, ..., 512/N\}$.
In terms of the rest deep methods, we search for hidden state size in $\{128, 256, 512\}$.
Besides, we also find $\lambda$ in $\{10^{-4}, 10^{-5}, 10^{-6}\}$.
We choose hyperparameters by the performance of the validation data based on the metrics in the following subsection. 

\subsection{Metrics}
Considering the root mean square error (RMSE), mean absolute percentage error (MAPE) and mean absolute error (MAE), we also leverage the glucose-specific RMSE, denoted as gRMSE \citep{DBLP:journals/tbe/FaveroFC12, DBLP:journals/npjdm/ZhuULHOG22}, to evaluate all the methods. 
The gRMSE
penalizes overestimation in
hypoglycemia and underestimation in hyperglycemia.
These two conditions can cause severe consequences to patients' health.
Besides, time lag is determined by analyzing the correlation between the forecasted BG levels and the actual readings from CGM.

\subsection{Comparison of prediction performance}
\begin{table}[tb]
		\renewcommand{\arraystretch}{1.3}
		\caption{Prediction of blood glucose (BG) levels in OhioT1DM. 
  }
	\label{tab:ohiot1dm}
		\centering
  \begin{threeparttable}
		\resizebox{0.5\textwidth}{!}{
\begin{tabular}{ccccccl}
\hline
Methods          & RMSE (mg/dL)                   & MAPE (\%)                      & MAE (mg/dL)                    & gRMSE (mg/dL)                  & \multicolumn{2}{c}{Time lag (min)}                \\ \hline
LR              & 22.19$\pm$0.00(2.79)$\ddagger$ & 10.90$\pm$0.00(2.13)$\ddagger$ & 15.92$\pm$0.00(1.95)$\ddagger$ & 27.69$\pm$0.00(3.72)$\ddagger$ & \multicolumn{2}{c}{8.34$\pm$0.00(6.06)$\dagger$}  \\
XGBoost         & 22.51$\pm$0.04(3.32)$\ddagger$ & 10.92$\pm$0.04(2.24)$\ddagger$ & 16.08$\pm$0.05(2.26)$\ddagger$ & 28.89$\pm$0.08(4.77)$\ddagger$ & \multicolumn{2}{c}{9.36$\pm$0.12(6.47)$\ddagger$} \\ \hline
RETAIN          & 20.30$\pm$0.08(2.64)$\ddagger$ & 9.78$\pm$0.04(1.81)$\ddagger$  & 14.41$\pm$0.03(1.75)$\ddagger$ & 25.48$\pm$0.16(3.47)$\ddagger$ & \multicolumn{2}{c}{7.39$\pm$0.09(4.95)}           \\
IMV-FULL        & 21.61$\pm$0.32(2.99)$\ddagger$ & 10.21$\pm$0.18(1.79)$\ddagger$ & 15.18$\pm$0.23(1.89)$\ddagger$ & 27.16$\pm$0.40(4.16)$\ddagger$ & \multicolumn{2}{c}{6.42$\pm$0.28(4.50)}           \\
IMV-TENSOR      & 20.15$\pm$0.03(2.77)$\ddagger$ & 9.54$\pm$0.02(1.82)$\ddagger$  & 14.00$\pm$0.02(1.75)$\ddagger$ & 25.42$\pm$0.07(3.79)$\ddagger$ & \multicolumn{2}{c}{7.57$\pm$0.11(4.86)$\dagger$}  \\
ETN-ODE         & 21.00$\pm$0.22(2.92)$\ddagger$ & 10.11$\pm$0.13(1.95)$\ddagger$ & 14.78$\pm$0.18(1.93)$\ddagger$ & 26.41$\pm$0.32(3.94)$\ddagger$ & \multicolumn{2}{c}{8.64$\pm$0.35(5.22)$\ddagger$} \\
ATT-T-LSTM      & 21.62$\pm$0.45(2.98)$\ddagger$ & 10.46$\pm$0.25(1.99)$\ddagger$ & 15.32$\pm$0.39(1.96)$\ddagger$ & 27.08$\pm$0.54(3.95)$\ddagger$ & \multicolumn{2}{c}{8.12$\pm$0.36(5.36)$\ddagger$} \\
ATT-F-LSTM      & 20.31$\pm$0.06(2.69)$\ddagger$ & 9.78$\pm$0.04(1.81)$\ddagger$  & 14.29$\pm$0.04(1.72)$\ddagger$ & 25.52$\pm$0.09(3.67)$\ddagger$ & \multicolumn{2}{c}{7.48$\pm$0.15(5.24)$\ast$}     \\ \hline
N-BEATS          & 20.15$\pm$0.05(2.56)$\ddagger$ & 9.62$\pm$0.03(1.77)$\ddagger$  & 14.11$\pm$0.04(1.68)$\ddagger$ & 25.31$\pm$0.07(3.37)$\ddagger$ & \multicolumn{2}{c}{7.98$\pm$0.12(5.23)$\ddagger$} \\
NHiTS           & 20.14$\pm$0.03(2.47)$\ddagger$ & 9.60$\pm$0.02(1.74)$\ddagger$  & 14.07$\pm$0.02(1.61)$\ddagger$ & 25.24$\pm$0.07(3.20)$\ddagger$ & \multicolumn{2}{c}{7.55$\pm$0.23(4.61)$\ast$}     \\ \hline
GAT+GRU (L=1)   & 19.03$\pm$0.07(2.40)           & 9.10$\pm$0.03(1.77)            & 13.37$\pm$0.03(1.65)           & 23.75$\pm$0.09(3.18)           & \multicolumn{2}{c}{6.24$\pm$0.14(4.45)}           \\
GAT+GRU (L=2)   & 19.08$\pm$0.04(2.38)           & 9.08$\pm$0.02(1.76)            & 13.37$\pm$0.02(1.64)           & 23.82$\pm$0.08(3.15)           & \multicolumn{2}{c}{6.19$\pm$0.25(4.51)}           \\
GATv2+GRU (L=1) & \textbf{18.97$\pm$0.06(2.43)}  & \textbf{9.07$\pm$0.01(1.78)}   & \textbf{13.34$\pm$0.02(1.68)}  & \textbf{23.65$\pm$0.10(3.21)}  & \multicolumn{2}{c}{\textbf{6.19$\pm$0.14(4.47)}}  \\
GATv2+GRU (L=2) & 19.11$\pm$0.15(2.45)           & 9.08$\pm$0.03(1.78)            & 13.38$\pm$0.06(1.68)           & 23.89$\pm$0.22(3.29)           & \multicolumn{2}{c}{6.30$\pm$0.14(4.75)}           \\ 
\hline
\end{tabular}

		}
    \begin{tablenotes}
   
    \item {\fontsize{5.5}{1.5}\selectfont $\ast p\leq 0.05 $;\ $\dagger p\leq 0.01 $;\ $\ddagger p\leq 0.005 $;\par}
    \item {\fontsize{5.5}{1.5}\selectfont Total historical timetamps is $T=48$; Prediction horizon is $H=6$; BG levels are sampled every $\delta t = 5\ {\rm min}$;\par}
    \item {\fontsize{5.5}{1.5}\selectfont RMSE: root mean square error; MAPE: mean absolute percentage error;  \par}
    \item {\fontsize{5.5}{1.5}\selectfont MAE: mean absolute error; gRMSE: glucose-specific RMSE;  \par}
    \item {\fontsize{5.5}{1.5}\selectfont The result is formatted as ``$mean\pm sd_1(sd_2)$''; \par}
    \item {\fontsize{5.5}{1.5}\selectfont $sd_1$ is the standard deviation after running the experiments four times by changing random seed; \par}
    \item {\fontsize{5.5}{1.5}\selectfont $sd_2$ is the standard deviation of the metric results across the participants. \par}
    \item {\fontsize{5.5}{1.5}\selectfont {\color{black}Explanation of methods can refer to section \ref{sec:baseline}.}
        
    }
    
    \end{tablenotes}
  \end{threeparttable}
\end{table}
\begin{table}[tb]
		\renewcommand{\arraystretch}{1.3}		
  	\caption{Prediction of blood glucose (BG) levels in ShanghaiT1DM.}
	\label{tab:shanghait1dm}
		\centering
    \begin{threeparttable}
		\resizebox{0.5\textwidth}{!}{

\begin{tabular}{ccccccl}
\hline
Methods          & RMSE (mg/dL)                    & MAPE (\%)                       & MAE (mg/dL)                     & gRMSE (mg/dL)                   & \multicolumn{2}{c}{Time lag (min)}               \\ \hline
LR              & 22.57$\pm$0.00(23.53)$\ast$     & 11.47$\pm$0.00(6.38)            & 15.81$\pm$0.00(13.32)           & 26.33$\pm$0.00(26.77)           & \multicolumn{2}{c}{2.50$\pm$0.00(5.59)}          \\
XGBoost         & 22.68$\pm$0.14(12.65)$\ddagger$ & 17.19$\pm$0.18(14.52)$\ddagger$ & 17.67$\pm$0.07(10.86)$\ddagger$ & 29.01$\pm$0.27(17.30)$\ddagger$ & \multicolumn{2}{c}{2.50$\pm$0.00(5.59)}          \\ \hline
RETAIN          & 16.25$\pm$0.94(5.93)$\ast$      & 10.77$\pm$0.71(6.73)$\dagger$   & 12.28$\pm$0.57(5.06)$\ddagger$  & 20.00$\pm$1.41(8.42)$\ast$      & \multicolumn{2}{c}{3.12$\pm$0.62(6.04)}          \\
IMV-FULL        & 13.63$\pm$0.10(2.48)            & 9.01$\pm$0.18(3.45)             & 10.57$\pm$0.12(2.22)            & 16.65$\pm$0.13(3.96)            & \multicolumn{2}{c}{2.50$\pm$0.00(5.59)}          \\
IMV-TENSOR      & 13.88$\pm$0.18(2.69)            & 9.26$\pm$0.23(3.64)             & 10.80$\pm$0.17(2.45)            & 16.91$\pm$0.23(4.10)            & \multicolumn{2}{c}{3.44$\pm$0.54(6.27)}          \\
ETN-ODE         & 15.38$\pm$0.23(3.35)$\ddagger$  & 10.19$\pm$0.29(4.18)$\ast$      & 11.91$\pm$0.21(3.04)$\ddagger$  & 19.15$\pm$0.33(5.14)$\ddagger$  & \multicolumn{2}{c}{3.75$\pm$0.00(6.50)}          \\
ATT-T-LSTM      & 14.03$\pm$0.15(3.03)            & 9.15$\pm$0.15(3.80)             & 10.82$\pm$0.14(2.77)            & 16.94$\pm$0.22(4.29)            & \multicolumn{2}{c}{2.50$\pm$0.00(5.59)}          \\
ATT-F-LSTM      & 14.31$\pm$0.21(2.97)            & 9.35$\pm$0.23(3.78)             & 10.95$\pm$0.14(2.62)            & 17.30$\pm$0.28(4.17)            & \multicolumn{2}{c}{1.56$\pm$0.54(4.51)}          \\ \hline
N-BEATS          & 14.60$\pm$0.27(3.01)            & 9.53$\pm$0.31(3.43)             & 11.36$\pm$0.26(2.61)            & 17.59$\pm$0.40(4.54)            & \multicolumn{2}{c}{2.81$\pm$0.54(5.82)}          \\
NHiTS           & 14.86$\pm$0.53(3.23)            & 9.56$\pm$0.85(3.31)             & 11.41$\pm$0.55(2.68)            & 18.00$\pm$1.00(4.76)            & \multicolumn{2}{c}{3.44$\pm$0.54(6.27)}          \\ \hline
GAT+GRU (L=1)   & 13.80$\pm$0.12(2.82)            & 8.93$\pm$0.07(3.70)             & 10.53$\pm$0.06(2.51)            & 16.69$\pm$0.14(4.49)            & \multicolumn{2}{c}{1.25$\pm$0.00(4.15)}          \\
GAT+GRU (L=2)   & 14.01$\pm$0.34(2.85)            & 9.11$\pm$0.27(3.75)             & 10.66$\pm$0.22(2.47)            & 17.07$\pm$0.47(4.56)            & \multicolumn{2}{c}{\textbf{0.62$\pm$0.62(2.07)}} \\
GATv2+GRU (L=1) & \textbf{13.62$\pm$0.22(2.78)}   & \textbf{8.74$\pm$0.33(3.54)}    & \textbf{10.38$\pm$0.21(2.43)}   & \textbf{16.44$\pm$0.37(4.48)}   & \multicolumn{2}{c}{1.88$\pm$0.62(4.87)}          \\
GATv2+GRU (L=2) & 13.98$\pm$0.34(2.95)            & 8.97$\pm$0.35(3.72)             & 10.60$\pm$0.26(2.54)            & 16.90$\pm$0.50(4.68)            & \multicolumn{2}{c}{1.25$\pm$0.00(4.15)}          \\ \hline

\end{tabular}
               
                		}
    \begin{tablenotes}

    \item {\fontsize{5.5}{1.5}\selectfont $\ast p\leq 0.05 $;\ $\dagger p\leq 0.01 $;\ $\ddagger p\leq 0.005 $;\par}
    \item {\fontsize{5.5}{1.5}\selectfont Total historical timetamps is $T=16$; Prediction horizon is $H=2$; BG levels are sampled every $\delta t = 15\ {\rm min}$;\par}
    \item {\fontsize{5.5}{1.5}\selectfont RMSE: root mean square error; MAPE: mean absolute percentage error;  \par}
    \item {\fontsize{5.5}{1.5}\selectfont MAE: mean absolute error; gRMSE: glucose-specific RMSE;  \par}
    \item {\fontsize{5.5}{1.5}\selectfont The result is formatted as ``$mean\pm sd_1(sd_2)$''; \par}
    \item {\fontsize{5.5}{1.5}\selectfont $sd_1$ is the standard deviation after running the experiments four times by changing random seed; \par}
    \item {\fontsize{5.5}{1.5}\selectfont $sd_2$ is the standard deviation of the metric results across the participants. \par}
    \item {\fontsize{5.5}{1.5}\selectfont {\color{black}Explanation of methods can refer to section \ref{sec:baseline}.}
        
    }
    
    \end{tablenotes}
      \end{threeparttable}
\end{table}

\begin{table}[tb]
		\renewcommand{\arraystretch}{1.3}
        \caption{Prediction of blood glucose (BG) levels in ArisesT1DM.}
	\label{tab:arisest1dm}
		\centering
        \begin{threeparttable}
  
		\resizebox{0.5\textwidth}{!}{
\begin{tabular}{ccccccl}
\hline
Methods          & RMSE (mg/dL)                   & MAPE (\%)                      & MAE (mg/dL)                    & gRMSE (mg/dL)                   & \multicolumn{2}{c}{Time lag (min)}                 \\ \hline
LR              & 25.17$\pm$0.00(7.54)$\ddagger$ & 12.14$\pm$0.00(3.45)$\ddagger$ & 18.20$\pm$0.00(5.01)$\ddagger$ & 31.93$\pm$0.00(10.19)$\ddagger$ & \multicolumn{2}{c}{11.67$\pm$0.00(6.32)$\ast$}     \\
XGBoost         & 24.84$\pm$0.02(5.63)$\ddagger$ & 12.07$\pm$0.02(3.18)$\ddagger$ & 17.99$\pm$0.02(3.85)$\ddagger$ & 32.32$\pm$0.03(7.76)$\ddagger$  & \multicolumn{2}{c}{13.37$\pm$0.32(7.24)$\ddagger$} \\ \hline
RETAIN          & 21.46$\pm$0.11(4.35)$\ddagger$ & 10.41$\pm$0.05(2.55)$\ddagger$ & 15.49$\pm$0.07(2.97)$\ddagger$ & 27.27$\pm$0.17(5.98)$\ddagger$  & \multicolumn{2}{c}{11.01$\pm$0.13(6.37)$\ast$}     \\
IMV-FULL        & 24.44$\pm$0.34(5.58)$\ddagger$ & 11.93$\pm$0.16(2.83)$\ddagger$ & 17.89$\pm$0.28(3.89)$\ddagger$ & 30.89$\pm$0.46(7.37)$\ddagger$  & \multicolumn{2}{c}{11.26$\pm$0.38(6.68)}           \\
IMV-TENSOR      & 21.48$\pm$0.23(4.57)$\ddagger$ & 10.31$\pm$0.04(2.52)$\ddagger$ & 15.35$\pm$0.09(2.95)$\ddagger$ & 27.41$\pm$0.33(6.27)$\ddagger$  & \multicolumn{2}{c}{10.78$\pm$0.09(6.35)}           \\
ETN-ODE         & 23.18$\pm$0.26(5.33)$\ddagger$ & 10.99$\pm$0.15(2.77)$\ddagger$ & 16.38$\pm$0.23(3.35)$\ddagger$ & 29.75$\pm$0.35(7.35)$\ddagger$  & \multicolumn{2}{c}{12.20$\pm$0.66(6.74)$\dagger$}  \\
ATT-T-LSTM      & 25.60$\pm$0.64(6.04)$\ddagger$ & 12.29$\pm$0.26(2.85)$\ddagger$ & 18.56$\pm$0.44(4.02)$\ddagger$ & 32.68$\pm$0.82(8.31)$\ddagger$  & \multicolumn{2}{c}{13.08$\pm$0.48(6.49)$\ddagger$} \\
ATT-F-LSTM      & 21.85$\pm$0.18(5.06)$\ddagger$ & 10.49$\pm$0.08(2.55)$\ddagger$ & 15.72$\pm$0.12(3.27)$\ddagger$ & 27.75$\pm$0.26(6.94)$\ddagger$  & \multicolumn{2}{c}{10.62$\pm$0.18(6.18)}           \\ \hline
N-BEATS          & 21.76$\pm$0.03(4.48)$\ddagger$ & 10.52$\pm$0.02(2.57)$\ddagger$ & 15.64$\pm$0.02(3.03)$\ddagger$ & 27.54$\pm$0.04(6.11)$\ddagger$  & \multicolumn{2}{c}{11.28$\pm$0.21(6.25)$\ast$}     \\
NHiTS           & 21.85$\pm$0.04(4.54)$\ddagger$ & 10.55$\pm$0.02(2.61)$\ddagger$ & 15.66$\pm$0.03(3.07)$\ddagger$ & 27.61$\pm$0.05(6.17)$\ddagger$  & \multicolumn{2}{c}{11.41$\pm$0.25(6.25)$\ast$}     \\ \hline
GAT+GRU (L=1)   & 20.02$\pm$0.12(3.94)           & 9.70$\pm$0.05(2.30)            & 14.50$\pm$0.07(2.70)           & 25.18$\pm$0.15(5.35)            & \multicolumn{2}{c}{9.57$\pm$0.45(5.22)}            \\
GAT+GRU (L=2)   & 20.00$\pm$0.11(3.91)           & \textbf{9.66$\pm$0.05(2.29)}            & 14.48$\pm$0.07(2.67)           & 25.15$\pm$0.14(5.28)            & \multicolumn{2}{c}{9.38$\pm$0.48(5.61)}            \\
GATv2+GRU (L=1) & \textbf{19.97$\pm$0.07(3.93)}  & 9.68$\pm$0.05(2.26)   & 14.47$\pm$0.05(2.69)  & \textbf{25.11$\pm$0.13(5.31)}   & \multicolumn{2}{c}{\textbf{9.53$\pm$0.28(5.26)}}   \\
GATv2+GRU (L=2) & 20.02$\pm$0.11(3.81)           & 9.67$\pm$0.05(2.28)            & \textbf{14.46$\pm$0.08(2.61)}           & 25.20$\pm$0.13(5.18)            & \multicolumn{2}{c}{9.81$\pm$0.08(5.95)}            \\ \hline
\end{tabular}
		}
     \begin{tablenotes}
    \item {\fontsize{5.5}{1.5}\selectfont $\ast p\leq 0.05 $;\ $\dagger p\leq 0.01 $;\ $\ddagger p\leq 0.005 $;\par}
    \item {\fontsize{5.5}{1.5}\selectfont Total historical timetamps is $T=48$; Prediction horizon is $H=6$; BG levels are sampled every $\delta t = 5\ {\rm min}$;\par}
    \item {\fontsize{5.5}{1.5}\selectfont RMSE: root mean square error; MAPE: mean absolute percentage error;  \par}
    \item {\fontsize{5.5}{1.5}\selectfont MAE: mean absolute error; gRMSE: glucose-specific RMSE;  \par}
    \item {\fontsize{5.5}{1.5}\selectfont The result is formatted as ``$mean\pm sd_1(sd_2)$''; \par}
    \item {\fontsize{5.5}{1.5}\selectfont $sd_1$ is the standard deviation after running the experiments four times by changing random seed; \par}
    \item {\fontsize{5.5}{1.5}\selectfont $sd_2$ is the standard deviation of the metric results across the participants. \par}
        \item {\fontsize{5.5}{1.5}\selectfont {\color{black}Explanation of methods can refer to section \ref{sec:baseline}.}
        
    }
        \end{tablenotes}
        \end{threeparttable}
\end{table}
\begin{table}[tb]
		\renewcommand{\arraystretch}{1.3}
         \caption{Prediction of blood glucose (BG) levels in ShanghaiT2DM.}
	\label{tab:shanghait2dm}
		\centering
  \begin{threeparttable}
		\resizebox{0.5\textwidth}{!}{
\begin{tabular}{ccccccl}
\hline
Methods          & RMSE (mg/dL)                    & MAPE (\%)                      & MAE (mg/dL)                    & gRMSE (mg/dL)                   & \multicolumn{2}{c}{Time lag (min)}               \\ \hline
LR              & 17.10$\pm$0.00(13.04)$\ddagger$ & 9.34$\pm$0.00(4.27)$\ddagger$  & 12.03$\pm$0.00(7.64)$\ddagger$ & 19.62$\pm$0.00(16.17)$\ddagger$ & \multicolumn{2}{c}{1.50$\pm$0.00(4.50)}          \\
XGBoost         & 16.75$\pm$0.02(5.64)$\ddagger$  & 11.09$\pm$0.04(7.58)$\ddagger$ & 12.79$\pm$0.03(5.22)$\ddagger$ & 20.12$\pm$0.04(7.41)$\ddagger$  & \multicolumn{2}{c}{1.09$\pm$0.06(4.43)}          \\ \hline
RETAIN          & 14.82$\pm$0.42(13.45)$\ddagger$ & 7.79$\pm$0.05(3.52)$\ddagger$  & 9.77$\pm$0.11(4.44)$\ddagger$  & 17.10$\pm$0.42(15.29)$\ddagger$ & \multicolumn{2}{c}{\textbf{0.53$\pm$0.48(8.11)}} \\
IMV-FULL        & 11.84$\pm$0.05(3.04)            & 6.97$\pm$0.07(2.32)            & 8.66$\pm$0.06(2.32)            & 13.67$\pm$0.07(3.91)            & \multicolumn{2}{c}{0.82$\pm$0.13(3.41)}          \\
IMV-TENSOR      & 12.14$\pm$0.13(3.15)$\ddagger$  & 7.29$\pm$0.12(2.79)$\ddagger$  & 8.96$\pm$0.08(2.48)$\ddagger$  & 14.11$\pm$0.12(4.08)$\ddagger$  & \multicolumn{2}{c}{1.50$\pm$0.00(4.50)$\ast$}    \\
ETN-ODE         & 12.90$\pm$0.30(3.58)$\ddagger$  & 7.66$\pm$0.17(3.01)$\ddagger$  & 9.44$\pm$0.20(2.71)$\ddagger$  & 15.03$\pm$0.34(4.63)$\ddagger$  & \multicolumn{2}{c}{1.05$\pm$0.18(3.81)}          \\
ATT-T-LSTM      & 13.92$\pm$0.58(11.99)$\ddagger$ & 7.41$\pm$0.04(2.73)$\ddagger$  & 9.34$\pm$0.11(4.09)$\ddagger$  & 15.91$\pm$0.61(13.07)$\ddagger$ & \multicolumn{2}{c}{0.90$\pm$0.00(3.85)}          \\
ATT-F-LSTM      & 12.76$\pm$0.56(5.53)$\ddagger$  & 7.29$\pm$0.01(2.59)$\ddagger$  & 9.08$\pm$0.08(2.84)$\ddagger$  & 14.72$\pm$0.53(6.37)$\ddagger$  & \multicolumn{2}{c}{1.05$\pm$0.00(3.83)}          \\ \hline
N-BEATS         & 12.15$\pm$0.03(3.19)$\ddagger$  & 7.12$\pm$0.05(2.09)$\ddagger$  & 8.90$\pm$0.03(2.44)$\ddagger$  & 14.08$\pm$0.05(4.06)$\ddagger$  & \multicolumn{2}{c}{1.35$\pm$0.00(4.29)}          \\
NHiTS           & 12.12$\pm$0.05(3.17)$\ddagger$  & 7.08$\pm$0.08(2.09)$\ddagger$  & 8.85$\pm$0.06(2.39)$\ddagger$  & 14.04$\pm$0.09(4.03)$\ddagger$  & \multicolumn{2}{c}{1.35$\pm$0.11(4.29)}          \\ \hline
GAT+GRU (L=1)   & 11.78$\pm$0.05(3.10)            & 6.97$\pm$0.05(2.36)$\ast$      & 8.63$\pm$0.05(2.37)$\ast$      & 13.62$\pm$0.06(3.96)            & \multicolumn{2}{c}{0.86$\pm$0.06(3.49)}          \\
GAT+GRU (L=2)   & 11.75$\pm$0.04(3.07)            & 6.98$\pm$0.04(2.47)            & 8.63$\pm$0.02(2.38)            & 13.60$\pm$0.06(3.93)            & \multicolumn{2}{c}{0.79$\pm$0.12(3.33)}          \\
GATv2+GRU (L=1) & \textbf{11.72$\pm$0.02(3.03)}   & \textbf{6.93$\pm$0.05(2.29)}   & \textbf{8.59$\pm$0.03(2.34)}   & \textbf{13.55$\pm$0.05(3.88)}   & \multicolumn{2}{c}{0.79$\pm$0.06(3.34)}          \\
GATv2+GRU (L=2) & 11.74$\pm$0.04(3.03)            & 7.00$\pm$0.05(2.60)            & 8.62$\pm$0.04(2.36)            & 13.58$\pm$0.04(3.87)            & \multicolumn{2}{c}{0.79$\pm$0.06(3.34)}          \\ \hline
\end{tabular}
		}

     \begin{tablenotes}

    \item {\fontsize{5.5}{1.5}\selectfont $\ast p\leq 0.05 $;\ $\dagger p\leq 0.01 $;\ $\ddagger p\leq 0.005 $;\par}
    \item {\fontsize{5.5}{1.5}\selectfont Total historical timetamps is $T=16$; Prediction horizon is $H=2$; BG levels are sampled every $\delta t = 15\ {\rm min}$;\par}
    \item {\fontsize{5.5}{1.5}\selectfont RMSE: root mean square error; MAPE: mean absolute percentage error;  \par}
    \item {\fontsize{5.5}{1.5}\selectfont MAE: mean absolute error; gRMSE: glucose-specific RMSE;  \par}
    \item {\fontsize{5.5}{1.5}\selectfont The result is formatted as ``$mean\pm sd_1(sd_2)$''; \par}
    \item {\fontsize{5.5}{1.5}\selectfont $sd_1$ is the standard deviation after running the experiments four times by changing random seed; \par}
    \item {\fontsize{5.5}{1.5}\selectfont $sd_2$ is the standard deviation of the metric results across the participants. \par}
    \item {\fontsize{5.5}{1.5}\selectfont {\color{black}Explanation of methods can refer to section \ref{sec:baseline}.}
        
    }
    \end{tablenotes}
    \end{threeparttable}
\end{table}
The evaluation of prediction is shown in Table \ref{tab:ohiot1dm}-\ref{tab:shanghait2dm}.
%
We also leveraged t-test and Wilcoxon test to evaluate the significance between ``GATv2+GRU ($L=1$)'' and other methods, where $p\leq 0.05$ means statistically significant.
We have some observations based on these tables.

1) Firstly, our proposed method, ``GATv2+GRU ($L=1$)'', outperforms all the baselines, while LR and XGboost perform worst.
%
%

2) Deep methods perform better than non-deep methods. Non-interpretable methods cannot promise better predicting performance compared with intepretable methods.

3)
Compared with ``GAT+GRU'', the dynamic scoring in ``GATv2+GRU'' can slightly improve the prediction performance in this scenario.
Hence, the modeling via dynamic scoring can bring limited advantages to BGLP.

\subsection{Interpretation of variable importance}
\begin{figure*}
	\centering
	\includegraphics[width = 2\columnwidth]{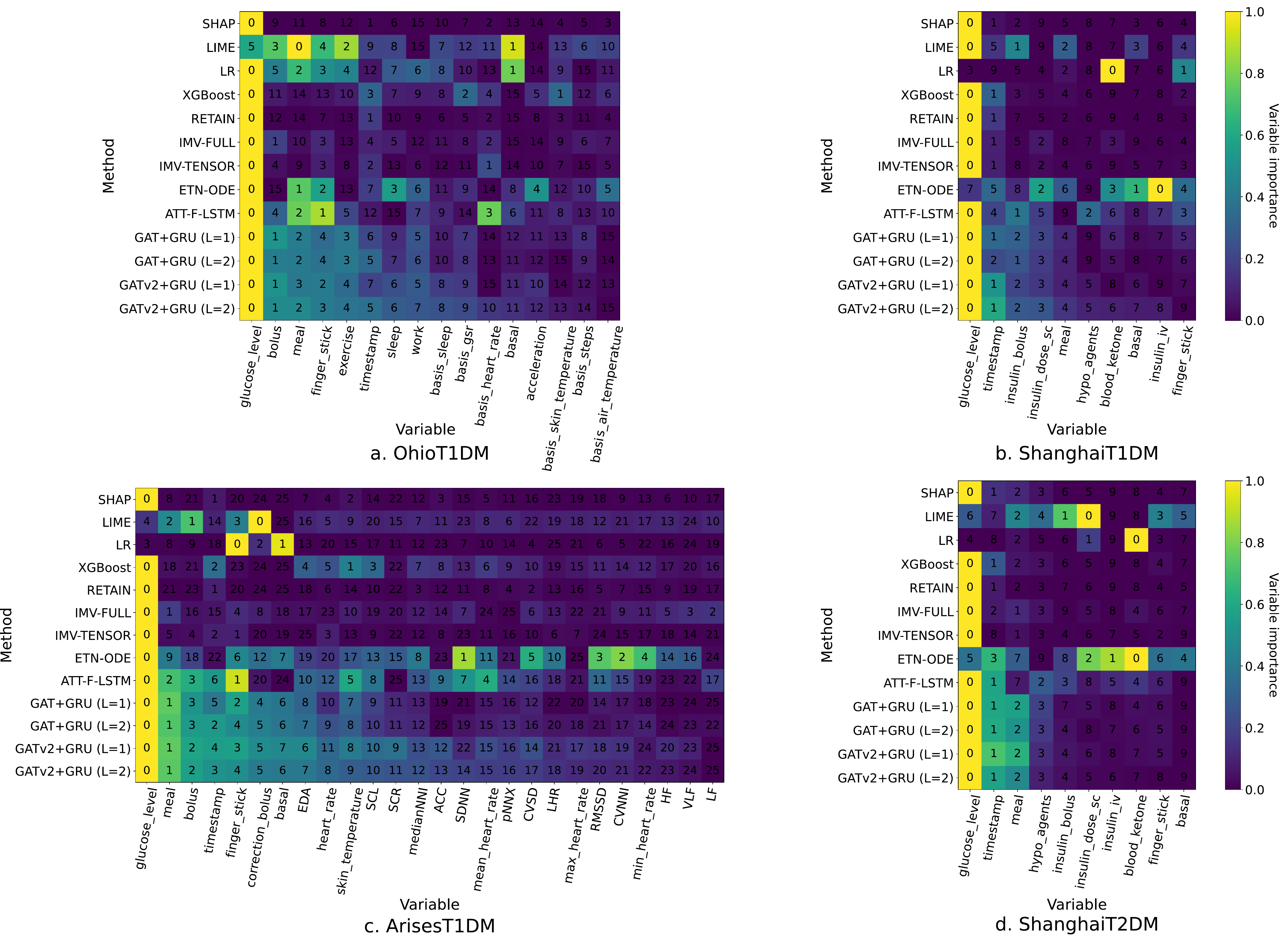}
 	\caption{The variable importance $v^j(\mathcal{I})$, scaled to $[0,1]$, of different methods, where $j$ is a variable from the variable set $\mathcal{N}$, and $\mathcal{I}$ contains all training examples. The number of graph layers is $L$. The x-axis/y-axis is variable/method. Brighter cell means higher $v^j(\mathcal{I})$. The number in a cell is the ranking place of the variable ranked by $v^j(\mathcal{I})$. {\color{black}Explanation of methods can refer to section \ref{sec:baseline}.}
  EDA: electrodermal activity; SCL: skin conductance level; SCR: 
 skin conductance response; medianNNI: median value of NN intervals; ACC: average 3D acceleration; SDNN: standard deviation of normal to normal (NN) intervals; pNNX: percentage of successive NN intervals greater than 50 ms; CVSD: coefficient of variation of successive differences; LHR: low-/high-frequency power ratio; RMSSD: root mean square of successive differences between adjacent NNs; CVNNI: coefficient of variation of NN intervals; HF: high frequency of heart rate in frequency domain; VLF: very high frequency of heart rate in frequency domain; LF: low frequency of heart rate in frequency domain.}
	\label{fig:variable_importance}
\end{figure*}

\begin{figure*}
	\centering
	\includegraphics[width = 2\columnwidth]{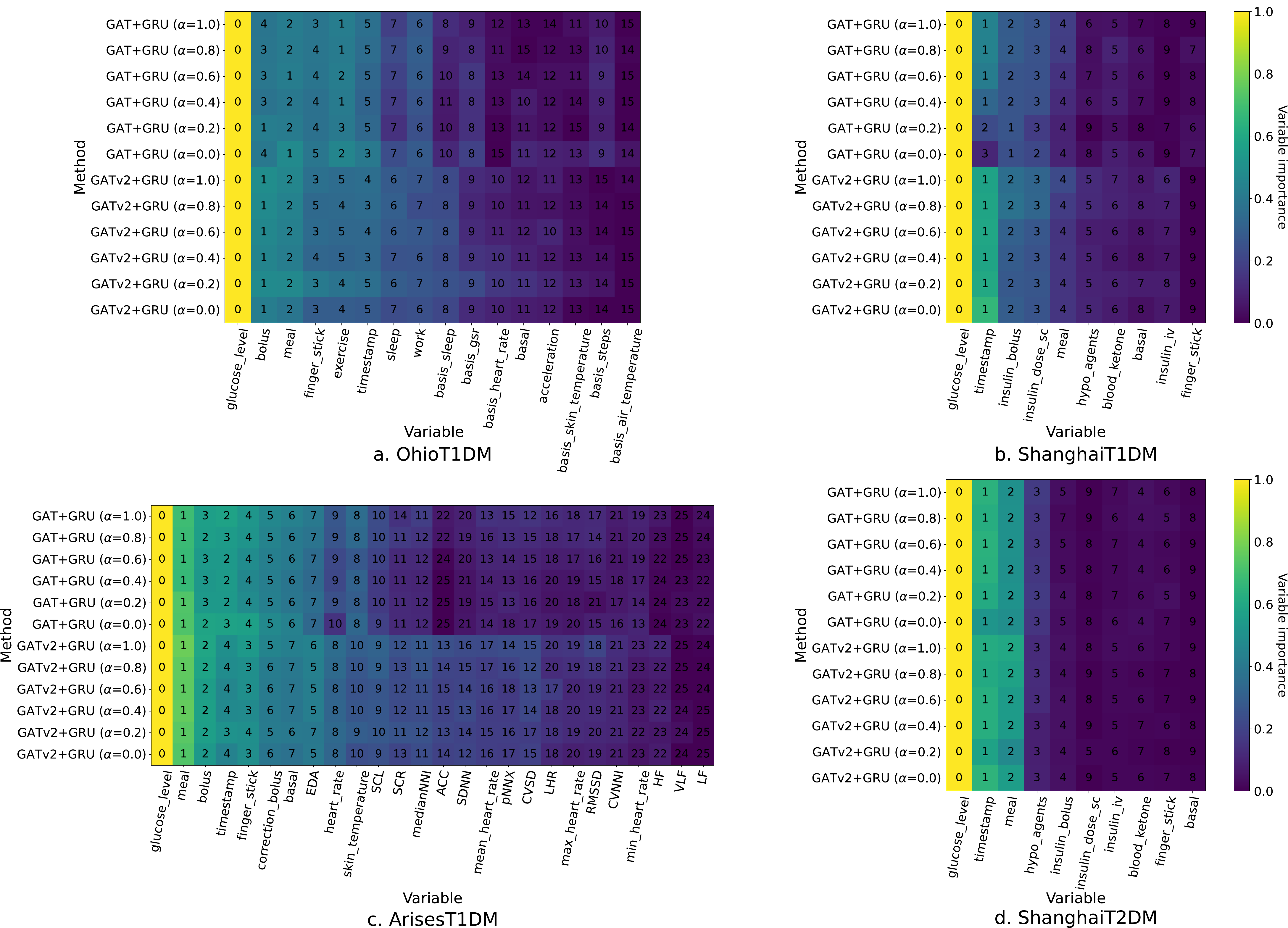}
  	\caption{The variable importance $v^j(\mathcal{I})$, scaled to $[0,1]$, of ``GAT+GRU'' and ``GATv2+GRU'' changes with different $\alpha$, where $\alpha$ is a hyperparameter of ${\rm LeakyReLU}(\cdot)$. Two graph layers are considered, and $j$ is a variable from the variable set $\mathcal{N}$, and $\mathcal{I}$ contains all training examples. The x-axis/y-axis is variable/method. Brighter cell means higher $v^j(\mathcal{I})$. The number in a cell is the ranking place of the variable ranked by $v^j(\mathcal{I})$. {\color{black}Explanation of methods can refer to section \ref{sec:baseline}.} EDA: electrodermal activity; SCL: skin conductance level; SCR: 
 skin conductance response; medianNNI: median value of NN intervals; ACC: average 3D acceleration; SDNN: standard deviation of normal to normal (NN) intervals; pNNX: percentage of successive NN intervals greater than 50 ms; CVSD: coefficient of variation of successive differences; LHR: low-/high-frequency power ratio; RMSSD: root mean square of successive differences between adjacent NNs; CVNNI: coefficient of variation of NN intervals; HF: high frequency of heart rate in frequency domain; VLF: very high frequency of heart rate in frequency domain; LF: low frequency of heart rate in frequency domain.}
	\label{fig:variable_importance_alpha}
\end{figure*}

After seeing the variable ranking lists w.r.t. $v^j(\mathcal{I})$ in four datasets in Figure \ref{fig:variable_importance}, we have the observations as follows:

1) The target variable, ``glucose\_level'', should be the most important in BGLP \citep{DBLP:conf/ecai/BevanC20, DBLP:journals/npjdm/ZhuULHOG22}. GARNNs, ATT-F-LSTM, IMV-LSTMs, RETAIN, XGBoost and SHAP can consistently focus on the target variable, while the rest methods often lose focus. 

%

2) GARNNs typically outperform baseline models by assigning significant importance to ``timestamp'' in the rankings, particularly for Shanghai T1DM and T2DM where the sample frequency of CGM is relatively low. This characteristic aligns with clinical observations that BG fluctuations are linked to personal lifestyles, exhibiting distinct temporal patterns.
It is significant that useful exogenous variables are highlighted when lacking enough values from the target variable in Shanghai datasets.

3)
Given that we predict the future BG of CGM instead of ``finger\_stick'', i.e., capillary BG test, more difference between them may reduce the importance of ``finger\_stick''.
%
%
%
The MAE between ``glucose\_level'' and ``finger\_stick'' in OhioT1DM and AriseT1DM is lower, i.e., 20.92 and 9.97 mg/dL, respectively.
The MAE between them in ShanghaiT1DM and ShanghaiT2DM are respectively 29.03 and 22.31 mg/dL.
Hence, GARNNs give special highlights to ``finger\_stick'' in OhioT1DM and AriseT1DM but less importance in ShanghaiT1DM and ShanghaiT2DM because of the larger MAE, while the variable importance ranking of other methods does not have similar observations.

4) Apart from ``glucose\_level'', ``finger\_stick'' and ``timestamp'', both bolus insulin (``bolus'', ``correction\_bolus'', ``insulin\_dose\_sc'' and ``insulin\_bolus'') and carbohydrate intake (``meal'') should be important variables as well, because they can rapidly affect the BG levels within short periods.
%
Basal insulin (``basal''), also called ``background insulin'', steadily controls BG levels for long periods.
Besides, ``insulin\_iv'' is an intravenous insulin infusion for super serious hyperglycemia under some extreme circumstances.

Therefore, in terms of daily cases, bolus insulin tends to be more important than basal insulin in BGLP.
GARNNs follow this knowledge in four datasets, while other methods fail to do that.
Besides, in ArisesT1DM, GARNNs give ``bolus'' more importance than ``correction\_bolus'', because the latter is an extra insulin taken during hyperglycemia.
The amount of the latter is much less than the former. 

On the other hand, a predominant reason for T2DM is that cells respond inactively to insulin.
Hence, compared with the variable ranking in ShanghaiT1DM, GARNNs reduce the importance of bolus insulin and give more importance to ``meal'' and non-insulin hypoglycemic agents, ``hypo\_agents'', in ShanghaiT2DM.

5) In OhioT1DM, the self-reported events, i.e., ``exercise'' and ``sleep'', should be more important than the sensor data, e.g., ``basis\_heart\_rate'', ``basis\_gsr'', etc.
This is because both exercising and sleeping can indirectly cause changes in BG levels.
Unlike the other methods, ``GATv2+GRU'' reflects this knowledge.
Besides, ``sleep'' and ``basis\_sleep'' are self-reported and sensor-detected, respectively, but ``basis\_sleep'' misses $56.89\%$ of the sleeping intensity data, so GARNNs give more importance to ``sleep''.

{\color{black}
6) In the analysis of the number of layers $L$ in GAT or GATv2, it is observed that this can slightly alter the ranking of variable importance $v^j(\mathcal{I})$. For instance, different configurations like ``GATv2+GRU($L=2$)'' show varying importance rankings for variables like ``sleep'' and ``work''. This was further substantiated by a feature ablation study in OhioT1DM \cite{DBLP:journals/corr/abs-2312-12541}, indicating that ``sleep'' improves prediction accuracy more than ``work''. The preference for ``sleep'' by ``GATv2+GRU($L=2$)'' seems more aligned with these findings, suggesting that multiple layers in GATv2 might capture more detailed interactions among variables. Similarly, there are subtle differences in rankings between GAT and GATv2, with GAT-based methods favoring ``work'' over ``sleep''. It is hypothesized that GAT's slightly lesser interpretability might be linked to its relatively weaker prediction performance, as indicated in tables \ref{tab:ohiot1dm}-\ref{tab:shanghait2dm}.

Despite these variations, both models consistently classify variables into categories like most important, very important, somewhat important, and less important.
For example, both models consistently identify ``glucose\_level'' as the most important variable across all datasets. This is followed by other closely related and significant variables that have a direct impact on BG levels, such as carbohydrate intake, classified as very important.
Subsequently, both models also pay attention to variables that indirectly affect BG levels, like exercise in the OhioT1DM dataset. These are considered somewhat important variables.
Finally, both GAT and GATv2 place other less significant variables at the lower end of the ranking.
}

7) Based on Theorem \ref{the:bounds}, the gaps between the prediction and the calculation of variable importance are affected by $\alpha$.
Figure \ref{fig:variable_importance_alpha} shows the variable importance ranking of GARNNs when changing $\alpha$.
We find that the gaps cannot significantly alter the variable importance, especially for ``GATv2+RNN''. Even with different gap sizes, the changes in variable importance remain small but still acceptable.

\subsection{Interpretation of feature maps}
\label{sec:featuremaps}

\begin{figure*}
	\centering
	\includegraphics[width = 2.0\columnwidth]{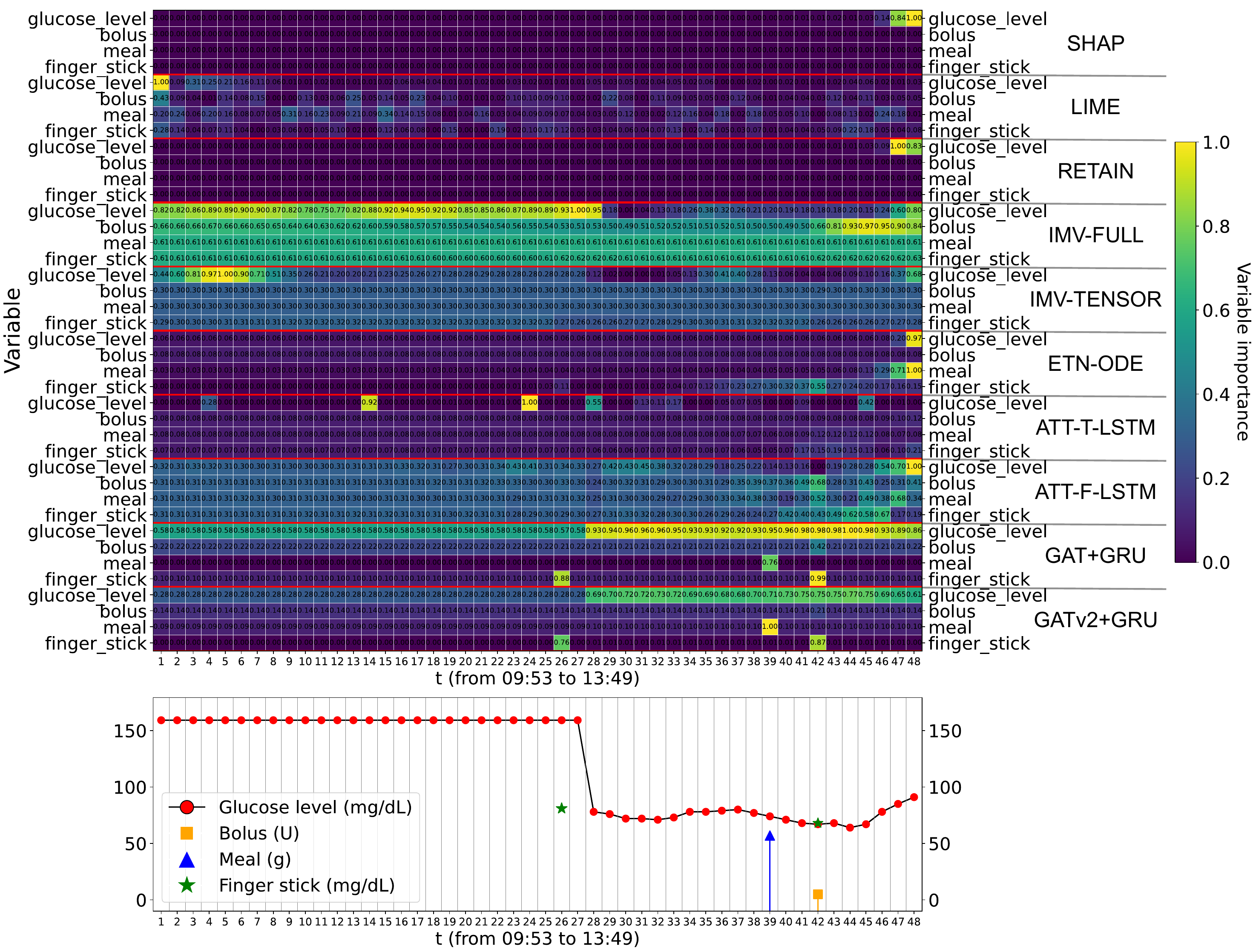}
	\caption{The bottom sub-figure is the visualization of a historical multi-variate time series of the patient 591 in OhioT1DM, only showing ``glucose\_level'', ``bolus'', ``meal'' and ``finger\_stick''. The heatmaps on the top are the feature maps of interpretable baseline methods and our proposed methods (``GAT+GRU'' and ``GATv2+GRU'' with two graph layers). The x-axis/y-axis is the timestep $t$ or variable name. The value in the cell is the variable importance $v_t^j$, scaled to $[0,1]$. {\color{black}Explanation of methods can refer to section \ref{sec:baseline}.}}
	\label{fig:feature_map_ohio}
\end{figure*}
\begin{figure*}[tb]
	\centering
	\includegraphics[width = 2\columnwidth]{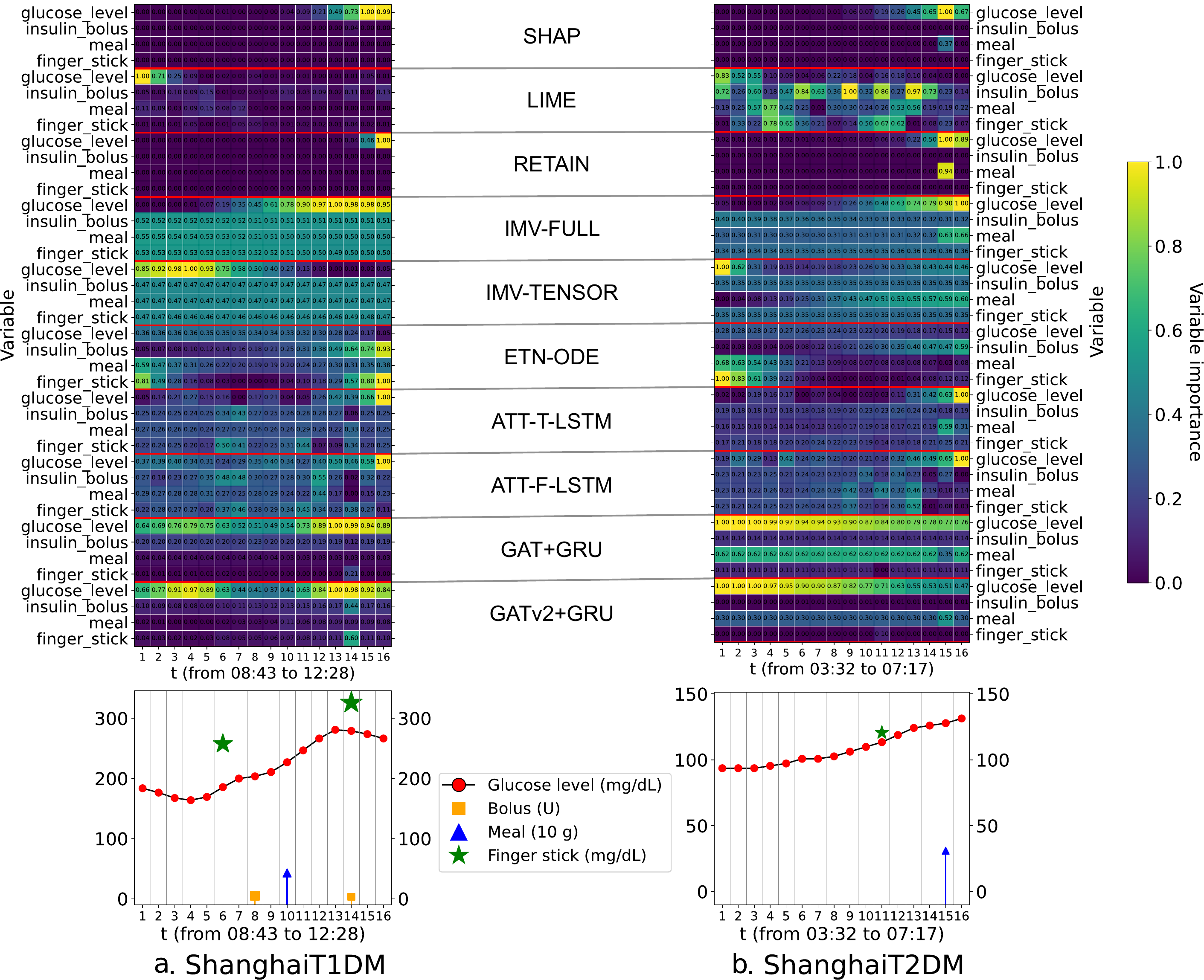}
	\caption{The bottom sub-figure is the visualization of a historical multi-variate time series of the patient 1001/2030 in ShanghaiT1DM/ShanghaiT2DM. The heatmaps on the top are the feature maps of interpretable baseline methods and our proposed methods (``GAT+GRU'' and ``GATv2+GRU'' with two graph layers). The x-axis/y-axis is the timestep $t$ or variable name. The value in the cell is the variable importance $v_t^j$, scaled to $[0,1]$. {\color{black} Explanation of methods can refer to section \ref{sec:baseline}.}}
	\label{fig:113_74}
\end{figure*}

CGM might be inaccurate after being worn for a period, so calibration by ``finger\_stick'' can make it back to normal.
According to the bottom sub-figure of Figure \ref{fig:feature_map_ohio}, participant 591 finds the CGM is unreliable, so this patient takes the first ``finger\_stick'' at $t=26$. Then, this patient calibrates  ``glucose\_level'' of CGM with ``finger\_stick'', and ``glucose\_level'' goes down and back to accurate measurements after $t=27$.
Next, this patient has a meal at $t=39$; this patient takes ``bolus'' and another ``finger\_stick'' at $t=42$.

Hence, ``glucose\_level'' ($t>27$) should be more important than ``glucose\_level'' ($t\leq27$).
Based on the feature maps of the methods, GARNNs accurately highlight the importance of ``glucose\_level'' ($t>27$).
Furthermore, GARNNs exactly capture the sparse signals, i.e., ``bolus'', ``meal'' and ``finger\_stick'', considering the related cells are markedly brighter when $t=26, 39 \ {\rm and}\ 42$.
Besides, GARNNs precisely focus on the lowest BG level, when $t=44$.
However, the feature maps of other methods are not quite informative. 
All of them cannot explicitly show the usage of sparse signals.
Meanwhile, they fail to focus ``glucose\_level'' time series reasonably. 

More feature maps are shown in Figure \ref{fig:113_74}.
Note that the ``meal'' in ShanghaiT1DM or ShanghaiT2DM means the amount of food intake rather than carbohydrate intake.
The variable importance is scaled to $[0, 1]$ as well.
The observations still hold that GARNNs ($L=2$) can correctly focus the important values of ``glucose\_level'', i.e., $t=4$ and $t=13$ in Figure \ref{fig:113_74}a, and $t=1$ in Figure \ref{fig:113_74}b.
Important sparse signals gain more attention as well, i.e.,  $t=14$ in Figure \ref{fig:113_74}a, and $t=15$ in Figure \ref{fig:113_74}b.

\section{Conclusion}
In this work, we propose GARNNs, novel interpretable models, by incorporating graph attention networks and RNNs for BGLP.
One notable advantage of GARNNs is their ability to provide significant variable importance for ranking variables and generating feature maps. 
In the experiments on four datasets, GARNNs demonstrate superior performance when compared with twelve baseline models. GARNNs outperform the baselines in both the accuracy of predictions and the quality of explanations regarding the contribution of variables.

{\color{black}
The main limitation of this work is its narrow focus, primarily concentrating on specific aspects of BG management without extensively exploring other relevant scenarios, such as the development of interpretable models for insulin advice. Consequently, future work will aim to expand the applicability of the proposed interpretable models across a broader range of BG management scenarios.
}

\section{Acknowledgement} 
We thank the support from UKRI Center for Doctoral Training in AI-enabled healthcare systems [EP/S021612/1] and University College London Overseas Research Scholarships.
{\color{black}
We also appreciate the support from Rosetrees Trust (Grant number: UCL-IHE-2020\textbackslash102) and Great Ormond Street Hospital (Charity ref.X12018).
The study sponsors had no study involvement.
}

{\color{black}

\bibliographystyle{elsarticle-num-names}
\bibliography{ref}


\end{document}